\documentclass[10pt,reqno,twoside]{article}

\usepackage[margin=0.8in]{geometry}

\title{Learning Exponential Family Graphical Models with Latent Variables using
Regularized Conditional Likelihood} % CORRECT TITLE
\author{Armeen Taeb $^{a}$, Parikshit Shah $^{b}$, and Venkat Chandrasekaran $^c$
\thanks{Correspondence email: armeen.taeb@stat.math.ethz.ch} \vspace{.25in} \\ $^a$ Department of Mathematics, ETH Z{\"u}rich \\  $^b$ Wisconsin Institutes for Discovery at the University of Wisconsin, Madison \\ $^c$  Department of Computing and Mathematical Sciences $\&$ of Electrical Engineering, Caltech}
\date{}
\makeatletter\@addtoreset{section}{part}\makeatother%

\newcommand{\upperRomannumeral}[1]{\uppercase\expandafter{\romannumeral#1}}

\renewcommand{\hat}{\widehat}
\usepackage{cite} 
\usepackage[numbers]{natbib}
\usepackage{siunitx}
\usepackage{parskip}
\usepackage{graphicx}
\usepackage{hyperref}

\newcommand{\Proj}{\mathcal{P}}

\newcommand{\R}{\mathbb{R}}

\newcommand{\Sym}{\mathbb{S}}

\renewcommand{\Delta}{\triangle}

\usepackage{amssymb,amsmath,amsthm}

\newtheorem{theorem}{Theorem}

\usepackage{enumitem}

\usepackage{verbatim}

\usepackage{titlesec}

\usepackage{titlesec}
\usepackage{lipsum}% just to generate text for the example

\numberwithin{equation}{section}

\titlespacing*{\section}
{0pt}{2ex}{1.1ex plus .2ex}\titlespacing{\subsubsection}{6pt}{\parskip}{-\parskip}\titlespacing*{\subsection}
{0pt}{2ex}{1.1ex plus .2ex}\titlespacing{\subsubsection}{6pt}{\parskip}{-\parskip}
\usepackage{titlesec}
\usepackage{placeins}
\usepackage{subfigure}
\date{October 19, 2020}

\begin{document}
\setlength{\abovedisplayskip}{15pt}
\setlength{\belowdisplayskip}{15pt}
\maketitle

\begin{abstract}
Fitting a graphical model to a collection of random variables given sample observations is a challenging task if the observed variables are influenced by latent variables, which can induce significant confounding statistical dependencies among the observed variables.  We present a new convex relaxation framework based on regularized conditional likelihood for latent-variable graphical modeling in which the conditional distribution of the observed variables conditioned on the latent variables is given by an exponential family graphical model.  In comparison to previously proposed tractable methods that proceed by characterizing the marginal distribution of the observed variables, our approach is applicable in a broader range of settings as it does not require knowledge about the specific form of distribution of the latent variables and it can be specialized to yield tractable approaches to problems in which the observed data are not well-modeled as Gaussian.  We demonstrate the utility and flexibility of our framework via a series of numerical experiments on synthetic as well as real data.\\[0.2in]
%\vspace{-0.1in}
% Please include a  of smaximumeven keywords
{\bf keywords:} convex optimization, equivariant estimators, exponential family PCA, pseudolikelihood, semidefinite programming
\end{abstract}

\section{Introduction}
Graphical models are multivariate statistical models that provide compact descriptions of joint probability distributions over large collections of variables in terms of products of local compatibility functions, each of which only involves a small number of the variables.  We consider an \emph{exponential family} of graphical models in $d$ variables in which the associated distributions factor as a product of functions of one or two variables (see  \cite{wrjordan} and the references therein) over a domain $\mathcal{X}^d \subseteq \R^d$ with ancillary statistic $h$:
\begin{equation} \label{eq:expfam}
\begin{aligned}
f(x; \alpha, \Theta) &\triangleq h(x) \exp\left\{\alpha' x - \tfrac{1}{2} x' \Theta x - \Phi(\alpha, \Theta) \right\} \\ \Phi(\alpha,\Theta) &\triangleq \int_{\mathcal{X}^d} \exp\left\{\alpha' x - \tfrac{1}{2} x' \Theta x \right\} h(x) d\nu(x).
\end{aligned}
\end{equation}
In the examples in this paper $\nu$ is either the Lebesgue measure or the counting measure (the integral in the definition of $\Phi$ is a sum if $\nu$ is the counting measure), and the product $h(x) d\nu(x)$ is also called the \emph{base measure}. The parameters $\alpha \in \R^d$ and $\Theta \in \Sym^d$ are the \emph{natural parameters} of the family, with $\Sym^d$ denoting the space of $d \times d$ real symmetric matrices.  The function $\Phi(\alpha, \Theta)$ is called the \emph{log-partition function} and it serves to normalize $f$.  The set of valid values for the parameters $\alpha, \Theta$ are those for which the log-partition is finite:
\begin{equation} \label{eq:validparam}
\mathcal{F} \triangleq \left\{ (\alpha, \Theta) \in \R^d \times \Sym^d ~|~ \Phi(\alpha, \Theta) < \infty \right\}.
\end{equation}
The set of \emph{valid parameters} $\mathcal{F}$ is a convex subset of $\R^d \times \Sym^d$, and over the domain $\mathcal{F}$ the log-partition function $\Phi$ is convex.  The number of parameters required to specify the distribution $f(x; \alpha, \Theta)$ is $\mathcal{O}(d^2)$, which can be prohibitive in problems with a large number of variables $d$.  Consequently, models in which $\Theta$ is a \emph{sparse} matrix are of great interest in applications.  Such sparse graphical models also have an appealing statistical interpretation as follows.  Given a distribution of the form \eqref{eq:expfam}, one can associate to it a graph consisting of $d$ nodes and edges between those pairs of nodes for which the corresponding $\Theta_{i,j} \neq 0$.  The Hammersley-Clifford theorem states that the random variables $x_i, x_j$ at two distinct nodes $i,j$ are independent conditioned on variables at all the other nodes if there is no edge between the nodes $i$ and $j$, i.e., there is no path between nodes $i$ and $j$ that does not pass through another node.  In this manner, the graph corresponding to the sparsity pattern of the matrix $\Theta$ encodes the conditional independence or Markov relations underlying the variables. 

Several variations of this basic family are possible, such as different types of compatibility functions or compatibility functions consisting of larger subsets of variables.  Our discussion can accommodate these extensions, but we stick with the model \eqref{eq:expfam} for notational simplicity.  Graphical models with random variables that are Gaussian ($\mathcal{X} = \R$) and Bernoulli ($\mathcal{X} = \{-1,+1\}$) are prominent examples of \eqref{eq:expfam}, and these are respectively called Gaussian graphical models and Ising models.  Further, graphical models that are appropriate for data specifying counts ($\mathcal{X} = \mathbb{Z}_+$) or data taking on positive values ($\mathcal{X} = \R_+$) may also be obtained as special cases of \eqref{eq:expfam} \cite{Besag,Genevera}; see Section~\ref{section:specialization} for details.

%via multivariate generalizations of the univariate exponential families associated to Poisson and exponential random variables, respectively

In data analysis problems in which the variables are indexed in an ordered fashion, there are usually reasonable choices for the underlying graph structure; for example, graphs based on nearest-neighbors are often used for specifying time series and spatial models.  However, in many applications, a natural choice for the graph structure is not available due to a lack of domain knowledge about the underlying conditional independence relations, and it is of interest to identify a sparse graphical model from sample observations of a collection of variables.  A challenge with this task is that there may be latent variables for which it is expensive or impossible to obtain sample observations.  Such unobserved variables pose a significant difficulty as graphical model structure is not closed under marginalization; therefore, the edge structure corresponding to the \emph{conditional} distribution of a collection of observed variables conditioned on latent variables is in general different from the \emph{marginal} distribution of the observed variables.  The graphical model of the observed variables conditioned on the latent variables signifies those statistical dependencies that are in some sense intrinsic to the observed variables, while the marginal graphical model of the observed variables consists of confounding dependencies that are induced due to marginalization over the latent variables, and this model typically consists of many more edges than the conditional graphical model.  In fact, even if the conditional graphical model is compactly described as a product of a small number of pairwise compatibility functions, the marginal model is in general much more complicated as it can consist of higher-order compatibility functions that link together large subsets of the observed variables; see Figure~\ref{fig:toy}. 

\FloatBarrier
\begin{figure}[ht!]
\centering
\begin{minipage}{0.3\textwidth}
\includegraphics[scale = 0.4]{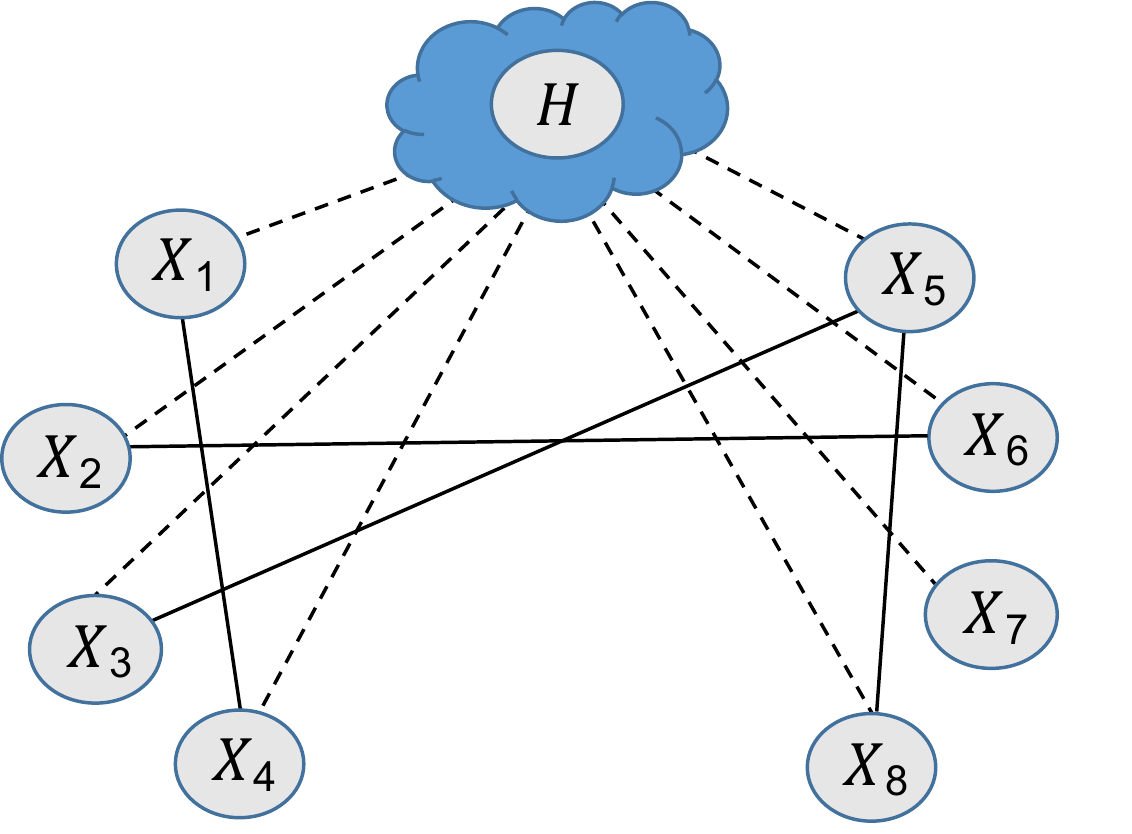}
\end{minipage}
\hspace{0.3in}
\begin{minipage}{0.3\textwidth}
\includegraphics[scale = 0.4]{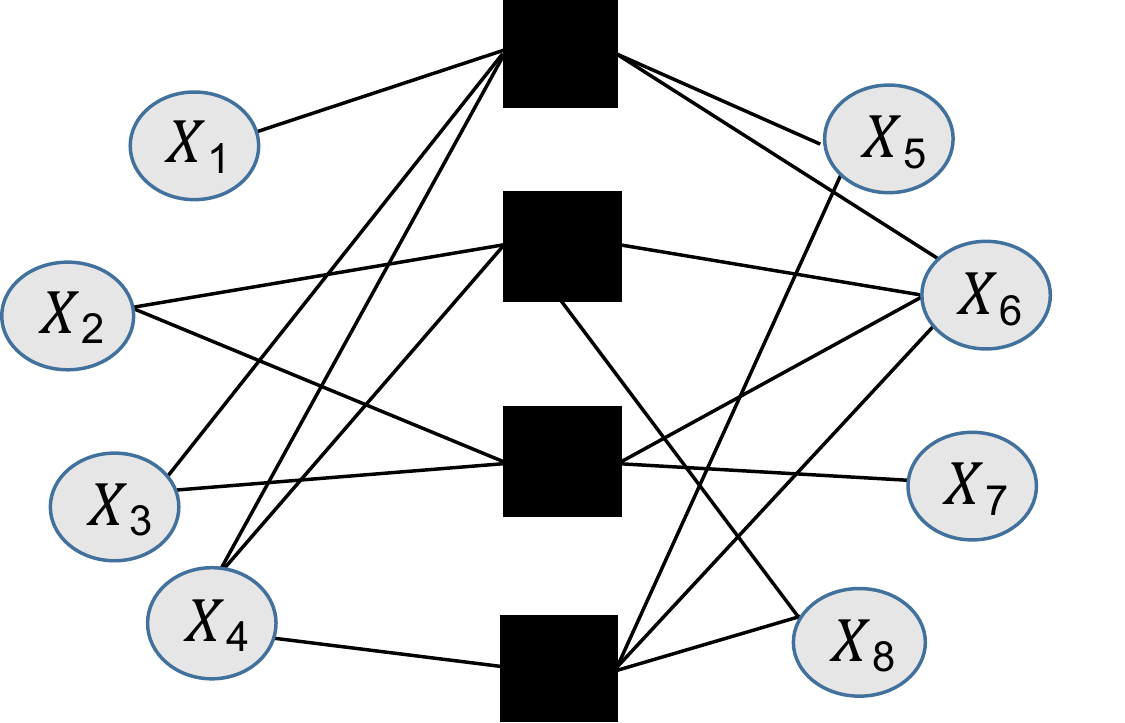}
\end{minipage}
\caption{An example of a graphical model over $8$ variables $X_1,\dots,X_8$ --- left: the variable $H$ represents an unobserved quantity, solid edges indicate pairwise interactions among observed variables, and dashed edges indicate links between observed and latent variables; right: factor graph where black squares represent higher-order interactions among variables that are linked through the factors.}
\label{fig:toy}
\end{figure}
\FloatBarrier
%\subsection{Prior Work and Outline of Our Approach}
The problem of learning a graphical model even without latent variables is computationally intractable in general, and accounting for the confounding effects of latent variables is more challenging.  There are a number of previous papers that propose computationally efficient approaches based both on combinatorial techniques \cite{bresler2,bresler1,ising_isinng} and on convex relaxation \cite{Chand2012,ising_gaussian} for learning graphical models with latent variables, along with theoretical or empirical demonstrations of the utility of these approaches for particular families of problem instances.  These methods proceed by studying the marginal distribution of the observed variables, and their derivation is based on an analysis of the structure of the confounding dependencies among the observed variables induced due to marginalization over the latent variables.  Consequently, the development of each of these methods is reliant on assumptions about the form of the joint distribution of the observed and latent variables -- jointly Gaussian observed and latent variables in \cite{Chand2012}, an Ising model specifying the observed and latent variables in \cite{bresler2,bresler1,ising_isinng}, and a conditionally Ising model for the observed variables with Gaussian latent variables in \cite{ising_gaussian}.

In this paper we present a new convex relaxation framework for latent-variable graphical model selection in which the conditional graphical model of the observed variables conditioned on the latent variables belongs to an exponential family of the form \eqref{eq:expfam}.  The virtues of convexity are by now self-evident -- tractable convex programs are manifestly implementable for moderate-sized problem instances based on off-the-shelf software packages, and in many cases, it is possible to develop special-purpose solvers that can scale to large instances by exploiting problem structure.  Perhaps the biggest conceptual distinction in our approach compared to those mentioned in the preceding paragraph is an analysis of the structure of the conditional graphical model of the observed variables conditioned on latent variables rather than the marginal distribution of the observed variables (which as explained previously can be very complicated in general).  Our viewpoint leads to tractable convex relaxations for a far broader class of models than those considered in previous work.  In particular, the derivation of our method does not require knowledge about the specific form of the distribution of the latent variables. In addition, our framework can also be specialized to settings in which the observed variables are not well-modeled as conditionally Gaussian or Bernoulli (such as with data specifying counts or taking on only positive values). In both these respects, our methodology is more broadly applicable than those described in the papers referenced above.

\subsection{Our Contributions} 
\vspace{0.05in}
We consider latent-variable graphical models in which a small number of latent variables $z \in \R^r$ influence the observed variables $x \in \mathcal{X}^d$, i.e., $r \ll d$, with the following form for the conditional distribution of $x | z$:
\begin{equation} \label{eq:lvgm}
f(x | z; \alpha, \Theta, B) \triangleq \exp\left\{(\alpha + Bz)' x -\tfrac{1}{2} x' \Theta x - \Phi(\alpha + Bz, \Theta) \right\} h(x).
\end{equation}
In words, the latent variables influence the natural parameters associated to the node compatibility functions in an affine manner specified by the parameter $B \in \R^{d \times r}$.  The form of this conditional distribution is akin to a generalized linear model with $x$ corresponding to the responses and $z$ the covariates, although there are two substantive differences -- one is that we do not observe $z$ and the other is that the different components of $x$ are not independent of each other in general after conditioning on $z$ (unless $\Theta_{i,j} = 0, ~ \forall i \neq j$) \cite{wrjordan}. The form of the conditional distribution \eqref{eq:lvgm} may also be interpreted as a conditional random field, but again with the distinction that we do not observe $z$ \cite{Genevera}.  The model \eqref{eq:lvgm} encompasses settings in which the joint distribution of the observed variables and the latent variables is given by a pairwise graphical model -- such as a Gaussian graphical model or an Ising model jointly over the observed and latent variables -- although our setup is more general as we do not assume a specific form for the distribution of the latent variables.  It is important to note that even if the conditional graphical model associated to $x | z$ is sparse (i.e., $\Theta$ is a sparse matrix), the marginal distribution of $x$ is in general dense depending on the distribution of $z$; more significantly, the conditional distribution of $x | z$ factorizes as a product of local functions (each depending on one or two variables), but the marginal distribution of $x$ may not be factorizable in such a local manner as the effect of marginalization over $z$ can induce confounding effects that couple all the components of $x$.

%With the following sufficient statistics:
%\begin{equation}
%\begin{aligned}
%\mu_n = \frac{1}{n} \sum_{k=1}^n x^{(k)}, ~~~~~ G_n = \frac{1}{n} \sum_{k=1}^n x^{(k)} {x^{(k)}}'.
%\end{aligned}
%\end{equation}

Our objective is to fit a latent-variable graphical model to a sample $\{x^{\scriptscriptstyle (k)}\}_{k=1}^n$ of size $n$ of the observed variables.  We assume that a selection of an exponential family that is best-suited to the data has been made (i.e., an appropriate choice of $\mathcal{X}$ and $h$, which in turn induce $\Phi$ and $\mathcal{F}$), and we consider the following regularized conditional likelihood optimization problem given user-specified regularization parameters $\lambda, \gamma \geq 0$:
\begin{equation}
\begin{aligned}
(\hat{\alpha}, \hat{\Theta}, \hat{L}) = \mathop{\arg\min}_{\substack{\alpha \in \R^d, \Theta \in \Sym^d \\ L \in \R^{d \times n}}} \, &\frac{1}{n} \sum_{k=1}^n \left[\Phi(\alpha+L^{\scriptscriptstyle (k)}, \Theta) - (\alpha + L^{\scriptscriptstyle (k)})' x^{\scriptscriptstyle (k)} + \tfrac{1}{2} {x^{\scriptscriptstyle (k)}}' \Theta x^{\scriptscriptstyle (k)} \right] + \lambda \|\Theta\|_{\ell_1} + \gamma \|L\|_\star \\ \mathrm{s.t.} \, &(\alpha+L^{\scriptscriptstyle (k)}, \Theta) \in \mathcal{F}, ~~~ k = 1,\dots,n.
\end{aligned}
\label{eq:clvgm}
\end{equation}
%\begin{equation}\label{eq:clvgm}
%\begin{aligned}
%(\hat{\alpha}, \hat{\Theta}, \hat{L}) = \mathop{\arg\min}_{\substack{\alpha \in \R^d, \Theta \in \Sym^d \\ L \in \R^{d \times n}}} ~ &-\frac{1}{n} \sum_{k=1}^n \log f(x^{(k)} | L^{(k)}) + \lambda \|\Theta\|_{\ell_1} + \gamma \|L\|_\star \\ \mathrm{s.t.} ~ &(\alpha+L^{(k)}, \Theta) \in \mathcal{F}, ~~~ k = 1,\dots,n.
%\end{aligned}
%\end{equation}
%Letting $X$ denote the $d \times n$ data matrix with columns $x^{(k)}$ and $1$ denote the all-ones vector of appropriate size, we note that
%\begin{equation*}
%-\frac{1}{n} \sum_{k=1}^n \log f(x^{(k)} | L^{(k)}) = \frac{1}{n} \sum_{k=1}^n \Phi(\alpha + L^{(k)}, \Theta) - \frac{1}{n} \mathrm{tr}((\alpha 1' + L)' X) + \frac{1}{2n}\mathrm{tr}(\Theta X X').
%\end{equation*}
The first part of the objective function represents the negative-logarithm of the conditional likelihood, with the vector $L^{\scriptscriptstyle (k)} \in \R^d$ denoting the $k$'th column of $L \in \R^{d \times n}$ and playing the role of $Bz^{\scriptscriptstyle (k)}$.  Both $B$ and $\{z^{\scriptscriptstyle (k)}\}_{k=1}^n$ are unknown, but the matrix with columns given by the elements of the set $\{Bz^{\scriptscriptstyle (k)}\}_{k=1}^n$ has small rank if the dimension $r$ of the latent vector satisfies $r \ll d$; the nuclear norm penalty $\|L\|_\star$ in the second line of the objective function is intended to promote this low-rank structure.  The $\ell_1$ penalty on $\Theta$ is useful for promoting sparsity of the conditional graphical model of the observed variables conditioned on the latent variables.  In summary, the optimization problem \eqref{eq:clvgm} is a convex program, although the log-partition function $\Phi$ may be intractable to compute in some cases (e.g., the conditional graphical model is an Ising model).  In such situations, one can appeal to further approximations from the literature which continue to preserve the convexity of the problem \cite{Besag,wrjordan}; see Section~\ref{section:non_gaussian}.  Finally, if we constrain the off-diagonal entries of the decision variable $\Theta$ in \eqref{eq:lvgm} to be zero, then we obtain a convex relaxation for an exponential-family generalization of principal components analysis (PCA) \cite{epca} in which one wishes to fit a model in which the different components of $x$ are independent of each other after conditioning on $z$ (i.e., $\Theta_{i,j} = 0, ~ \forall i \neq j$).

%Fitting to models of the form \eqref{eq:lvgm} in which the different components of $x$ are independent of each other after conditioning on $z$ (i.e., $\Theta_{i,j} = 0, ~ \forall i \neq j$) represents an exponential family generalization of principal components analysis (PCA) \cite{epca}, and our methodology can be specialized to yield a convex relaxation for this problem. 

%and in particular no prior knowledge of number $h$ or of the domain $\mathcal{Z}$ of the latent variables is required; if information about $\mathcal{Z}$ is available, e.g., $\mathcal{Z} = \{-1,+1\}$, then the nuclear norm regularizer can be strengthened to obtain a tighter relaxation (see Section[REF])

In Section~\ref{section:specialization} we specialize the formulation \eqref{eq:clvgm} to obtain computationally tractable methods for latent-variable graphical modeling for a range of exponential family graphical models, as well as a symmetry reduction of \eqref{eq:clvgm} for Gaussian models based on equivariance, which yields a convex program involving $d \times d$ matrices; in contrast, the complexity of solving \eqref{eq:clvgm} scales with the number of observations $n$ for general models.  In Section~\ref{sec:model_sel} we present a technique for selecting suitable regularization parameters $\lambda, \gamma$ \eqref{eq:clvgm}.  In each of these preceding sections, we provide evidence for the effectiveness of our framework via experiments on synthetic data.  In Section~\ref{sec:real}, we demonstrate the performance of our methods on real data.  Finally, we conclude with a discussion of future directions in Section~\ref{section:discussion}. 

The focus of this paper is on developing a mathematically principled and broadly applicable methodology for latent-variable graphical modeling.  We demonstrate the utility and flexibility of our framework empirically on synthetic and real data.  We describe a number of questions for future work that concern theoretical analysis of various aspects of our approach in Section~\ref{section:discussion}.

\subsection{Notation}
 We denote the identity matrix by $I$, with the size being clear from context.  The collection of positive-semidefinite matrices in $\mathbb{S}^d$ is denoted $\mathbb{S}^d_+$ and the collection of positive-definite matrices by $\mathbb{S}^d_{++}$.

\section{Specializations of Our Framework}
\label{section:specialization}
We describe specializations of our framework to various exponential family models.  Section~\ref{section:gaussian} concerns Gaussian models in which the log-partition function can be computed efficiently, and therefore our proposed method is simply the specialization of the convex relaxation \eqref{eq:clvgm} to the Gaussian case. Furthermore, we show that one can equivalently reformulate the Gaussian specialization of \eqref{eq:clvgm} as an SDP involving only $d \times d$ matrix decision variables, so that no decision variable has a dimension that scales with $n$. Section~\ref{section:non_gaussian} concerns several non-Gaussian models for which the likelihood is intractable to compute in general, and therefore computationally efficient approximations of \eqref{eq:clvgm} are required; we describe one such approach using the pseudo-likelihood approximation of Besag \cite{Besag}.  In both subsections, we discuss how previous approaches for graphical modeling without latent variables and for exponential family PCA may be obtained by restricting our relaxations appropriately.  Finally, we present numerical evidence for the effectiveness of our methods in Section~\ref{section:phase}.
\subsection{Gaussian Models} 
\label{section:gaussian}
Multivariate Gaussians constitute an exponential family of distributions with ancillary statistic $h \equiv 1$ and $\nu$ being the Lebesgue measure.  The parameter $\Theta$ is the precision or inverse covariance matrix and the mean is given by $\Theta^{\scriptscriptstyle -1} \alpha$.  The corresponding log-partition function and valid parameters are given by:
\begin{equation} \label{eq:gaussianexpfam}
\begin{aligned}
\Phi_{\mathrm{gaussian}}(\alpha,\Theta) &= \tfrac{1}{2} \left(\alpha' \Theta^{\scriptscriptstyle  -1} \alpha - \log\det \Theta + d \log 2 \pi \right) \\  \mathcal{F}_{\mathrm{gaussian}} &= \{(\alpha, \Theta) \in \R^d \times \Sym^d ~|~ \Theta \succ 0\}.
\end{aligned}
\end{equation}
Thus, we obtain the following convex relaxation for latent-variable Gaussian graphical modeling given data $\{x^{\scriptscriptstyle (k)}\}_{k=1}^n \subset \R^d$ and user-specified regularization parameters $\lambda, \gamma \geq 0$:
\begin{equation} \label{eq:gaussianclvgm}
\begin{aligned}
(\hat{\alpha}, \hat{\Theta}, \hat{L}) = \mathop{\arg\min}_{\substack{\alpha \in \R^d, \Theta \in \Sym^d \\ L \in \R^{d \times n}}} ~ &\frac{1}{n} \sum_{k=1}^n \left[\tfrac{1}{2}(\alpha + L^{\scriptscriptstyle (k)})' \Theta^{\scriptscriptstyle -1} (\alpha + L^{\scriptscriptstyle (k)}) - (\alpha + L^{\scriptscriptstyle (k)})' x^{\scriptscriptstyle (k)} + \tfrac{1}{2} {x^{\scriptscriptstyle (k)}}' \Theta x^{\scriptscriptstyle (k)} \right] \\ & -\frac{1}{2} \log\det \Theta + \lambda \|\Theta\|_{\ell_1} + \gamma \|L\|_\star \\ \mathrm{s.t.} ~ & \Theta \succ 0.
\end{aligned}
\end{equation}
Sublevel sets of the function $(\alpha + L^{(k)})' \Theta^{\scriptscriptstyle -1} (\alpha + L^{(k)})$ may be expressed via Schur complements, and therefore, this problem is a log-determinant semidefinite program (SDP) that can be solved to a desired precision in polynomial-time.   % In Section[REF] we compare Equation \eqref{eq:gaussianclvgm} to a previous SDP relaxation for latent-variable Gaussian graphical modeling that is based on regularized marginal likelihood [CITE].

If we do not account for the confounding effects of latent variables and set $L = 0$, then the convex program \eqref{eq:gaussianclvgm} specializes to the well-known `Graphical Lasso' method \cite{BEA:08,Glasso,Yuan}, which corresponds to $\ell_1$-regularized marginal log-likelihood.  On the other hand, if we restrict $\Theta$ to be a diagonal matrix we recover a convex relaxation for factor analysis \cite{Shapiro}.

%The parameter $\Theta$ is the precision or inverse covariance matrix and the mean is given by $\Theta^{-1} \alpha$.  The valid parameters are those in which the precision matrix is positive definite:
%\begin{equation} \label{eq:gaussianvalid}
%\mathcal{F}_{\mathrm{gaussian}} = \{(\alpha, \Theta) \in \R^d \times \Sym^d ~|~ \Theta \succ 0\}.
%\end{equation}

%\subsection{Symmetry Reduction via Equivariance for Gaussian Models}
%\label{section:equivariance}
The dimension of the decision variable $L$ in the convex relaxation \eqref{eq:gaussianclvgm} grows with the number of observations $n$.  Consequently, the computational runtime to solve \eqref{eq:gaussianclvgm} to a desired accuracy scales polynomially with $n$. We exploit an equivariance property underlying the estimator \eqref{eq:gaussianclvgm} in the Gaussian case to obtain an equivalent relaxation (after preprocessing) in which the dimensions of the decision variables do not depend on $n$.  (The main component of the preprocessing step is a singular value decomposition of a $d \times n$ matrix, but the complexity of this operation scales more modestly with $n$ than that of solving \eqref{eq:gaussianclvgm}.)  For ease of analysis, we set $\alpha = 0$ \eqref{eq:gaussianclvgm}; this restriction may be made with no loss of generality if we center the observations prior to solving \eqref{eq:gaussianclvgm}.

%\subsection{Reduction via Equivariance} 
Formally, letting $X \in \R^{d \times n}$ denote a data matrix with the observations $\{x^{(k)}\}_{k=1}^n \subset \R^d$ specifying the columns, the objective \eqref{eq:gaussianclvgm} may be written as follows:
\begin{equation} \label{eq:gaussianobj}
c(\Theta, L; X) = \tfrac{1}{2n} \mathrm{tr}(L' \Theta^{\scriptscriptstyle -1} L) -\tfrac{1}{2}\log\det(\Theta) - \tfrac{1}{n}\mathrm{tr}(L'X) + \tfrac{1}{2n}\mathrm{tr}(\Theta X X') + \lambda \|\Theta\|_{\ell_1} + \gamma \|L\|_\star.
\end{equation}
A key attribute of this expression is that for any matrix $W \in \R^{m \times n}, m \geq n,$ satisfying $W' W = I$, one can check that:
\begin{equation} \label{eq:gaussianequiv}
    c(\Theta, LW'; XW') = c(\Theta, L; X).
\end{equation}
The dimensions of the inputs of $c$ here are, in general, different on the left-hand-side versus the right-hand-side, but the expression \eqref{eq:gaussianobj} remains valid as long as the dimensions of the inputs/parameters to $c$ are consistent; we allow for this flexibility in our discussion in the sequel.  Hence, if the data matrix $X$ is transformed as $X \leftarrow X W'$ then the expression \eqref{eq:gaussianobj} remains unchanged with an analogous transformation $L \leftarrow L W'$ applied to the decision variable $L$ (and leaving $\Theta$ unchanged).  This \emph{equivariance} property enables a reduction in the size of the SDP \eqref{eq:gaussianclvgm}, which we formalize via the following result:

\begin{theorem}
Given a data matrix $X \in \R^{d \times n}$, let $\Sigma = \tfrac{1}{n} X X'$ denote the sample covariance matrix and consider the following optimization problem\footnote{To ensure that this problem has an optimal solution, it suffices to choose $\lambda > 0$ or to have $\Sigma \succ 0$.} with $\lambda, \gamma \geq 0$:
\begin{equation} \label{eq:gaussianreduced}
\begin{aligned}
(\hat{\Theta},\hat{H}) = \mathop{\arg\min}_{\substack{\Theta \in \Sym^d_{++} \\ H \in \R^{d \times d}}} ~ \tfrac{1}{2}\mathrm{tr}(H' \Theta^{\scriptscriptstyle -1} H) - \tfrac{1}{2} \log\det \Theta - \mathrm{tr}(H' \sqrt{\Sigma}) + \tfrac{1}{2}\mathrm{tr}(\Theta \Sigma) + \lambda \|\Theta\|_{\ell_1} + \gamma \sqrt{n} \|H\|_\star
\end{aligned}
\end{equation}
Here $\sqrt{\Sigma}$ denotes the positive-semidefinite square root of $\Sigma$.  Let $X = U D V'$ be a singular value decomposition of $X$.  Then $(\hat{\Theta}, \sqrt{n} \hat{H} U V')$ is an optimal solution of \eqref{eq:gaussianclvgm} with $\alpha = 0$.
\end{theorem}

\begin{proof}
Consider the case $n \leq d$ in which $U \in \R^{d \times n}, V \in \R^{n \times n}$.  One can check that:
\begin{equation*}
    c(\Theta, L; X) = c(\Theta, LVU'; XVU') = c(\Theta, LVU'; \sqrt{n} \sqrt{\Sigma})
\end{equation*}
The first equality follows from the equivariance relation \eqref{eq:gaussianequiv} and the second equality follows from the definition of $\Sigma$.  Setting $H = \tfrac{1}{\sqrt{n}} L V U'$, the expression $c(\Theta, LVU'; \sqrt{n} \sqrt{\Sigma})$ is equal to the objective of \eqref{eq:gaussianreduced}.  Hence, a feasible $(\Theta,L)$ for \eqref{eq:gaussianclvgm} with $\alpha = 0$ leads to a feasible point $(\Theta,H)$ for \eqref{eq:gaussianreduced} with equal cost.  In the other direction, consider a feasible $(\Theta,H)$ for \eqref{eq:gaussianreduced}, and set $L = \sqrt{n} H U V'$.  With this $(\Theta,L)$, we consider the three terms of $c(\Theta,L)$ from \eqref{eq:gaussianobj} involving $L$.  First, we note using cyclicity of trace that:
\begin{equation*}
    \tfrac{1}{2n} \mathrm{tr}(L' \Theta^{\scriptscriptstyle -1} L) = \tfrac{1}{2} \mathrm{tr}(V U' H' \Theta^{\scriptscriptstyle -1} H U V') = \tfrac{1}{2} \mathrm{tr}(H' \Theta^{\scriptscriptstyle -1} H U U') \leq \tfrac{1}{2} \mathrm{tr}(H' \Theta^{\scriptscriptstyle -1} H),
\end{equation*}
which follows from $V'V = I$ and $I \succeq U U'$.  Next, we have that:
\begin{equation*}
    \tfrac{1}{n} \mathrm{tr}(L'X) = \tfrac{1}{\sqrt{n}} \mathrm{tr}(V U' H' X) = \tfrac{1}{\sqrt{n}} \mathrm{tr}(H' X V U') = \mathrm{tr}(H' \sqrt{\Sigma})
\end{equation*}
using the definition of $\Sigma$ and the cyclicity of trace.  Finally, we observe that:
\begin{equation*}
    \gamma \|L\|_\star = \gamma \sqrt{n} \| H U V'\|_\star \leq \gamma \sqrt{n} \|H\|_\star \|U V'\|_2 \leq \gamma \sqrt{n} \|H\|_\star
\end{equation*}
by applying the H\"older inequality to the nuclear norm.  Therefore, for any feasible $(\Theta,H)$ for \eqref{eq:gaussianreduced}, the point $(\Theta,\sqrt{n} H U V')$ is feasible for \eqref{eq:gaussianclvgm} with $\alpha = 0$ and has equal or lower cost.

% This concludes the proof of the case in which $n \leq d$.

Consider next the case $n > d$ in which $U \in \R^{d \times d}, V \in \R^{n \times d}$.  Set $W = \begin{pmatrix} U V' \\ P \end{pmatrix} \in \R^{n \times n}$ with $P \in \R^{(n-d) \times n}$ so that $WW' = W'W = I$.  One can then check that:
\begin{equation*}
    c(\Theta,L; X) = c(\Theta, L W'; X W') = c(\Theta, (LVU' ~~ LP'); (\sqrt{n} \sqrt{\Sigma} ~~ 0))
\end{equation*}
The first equality follows from the equivariance relation \eqref{eq:gaussianequiv}, and the second equality follows from the properties of $W$ and the definition of $\Sigma$.  Given a feasible $(\Theta,H)$ of \eqref{eq:gaussianreduced} and setting $L = \sqrt{n} H U V'$, we observe that $L V U' = \sqrt{n} H$ and $L P' = 0$.  Thus, with this choice of $L$ the expression $c(\Theta, L; X) = c(\Theta, (\sqrt{n} H ~~ 0); (\sqrt{n} \sqrt{\Sigma} ~~ 0))$ equals the objective of \eqref{eq:gaussianreduced}.  Hence, from a feasible point of \eqref{eq:gaussianreduced}, we obtain a feasible point of \eqref{eq:gaussianclvgm} with equal cost.  In the other direction, let $(\Theta,L)$ be feasible for \eqref{eq:gaussianclvgm} and set $H = \tfrac{1}{\sqrt{n}} L V U'$.  We consider the three terms of the objective from \eqref{eq:gaussianreduced} involving $H$.  First, we note using the cyclicity of trace that:
\begin{equation*}
    \tfrac{1}{2} \mathrm{tr}(H' \Theta^{\scriptscriptstyle -1} H) = \tfrac{1}{2n} \mathrm{tr}(U V' L' \Theta^{\scriptscriptstyle -1} L V U') = \tfrac{1}{2n} \mathrm{tr}( L' \Theta^{\scriptscriptstyle -1} L V V') \leq \tfrac{1}{2n} \mathrm{tr}( L' \Theta^{\scriptscriptstyle -1} L),
\end{equation*}
which follows from $U'U = I$ and $I \succeq V V'$.  Next, we have that:
\begin{equation*}
    \mathrm{tr}(H' \sqrt{\Sigma}) = \tfrac{1}{\sqrt{n}}\mathrm{tr}(U V' L' \sqrt{\Sigma}) = \tfrac{1}{n} \mathrm{tr}(L' X),
\end{equation*}
using the definition of $\Sigma$ and the cyclicity of trace.  Finally, we observe that:
\begin{equation*}
    \gamma \sqrt{n} \|H\|_\star = \gamma \|L V U'\|_\star \leq \gamma \|L\|_\star \|VU'\|_2 \leq \gamma \|L\|_\star,
\end{equation*}
by applying the H\"older inequality to the nuclear norm.
Thus, for any feasible $(\Theta,L)$ for \eqref{eq:gaussianclvgm} with $\alpha = 0$, the point $(\Theta,\tfrac{1}{\sqrt{n}}LVU')$ is feasible for \eqref{eq:gaussianreduced} and has lower or equal cost.  %This concludes the proof of the case $n > d$.
\end{proof}
Although this result holds for arbitrary $n$, it is most relevant when $n \gg d$.
\subsection{Non-Gaussian Models}
\label{section:non_gaussian}
\vspace{0.05in}
We consider three exponential family graphical models that are relevant in settings in which the observations are not well-modeled as Gaussian.  These are derived by considering pairwise graphical models in which the variable at each node conditioned on the variables at all the other nodes is distributed according to a Bernoulli, Poisson, or exponential random variable \cite{Besag,Genevera}; as such the following models represent natural multivariate generalizations of popular univariate exponential families:

\noindent \paragraph{Ising models} Here $\mathcal{X} = \{-1,+1\}$, the ancillary statistic is $h \equiv 1$, and $\nu$ is the counting measure.  As the log-partition function is given by a finite sum, the set of valid parameters is not constrained in a significant way:
\begin{equation} \label{eq:isingvalid}
\mathcal{F}_{\mathrm{ising}} = \{(\alpha, \Theta) \in \R^d \times \Sym^d ~|~ \Theta_{i,i} = 0, ~ i=1,\dots,d \}.
\end{equation}
The condition on the diagonal elements of $\Theta$ is due to the fact $x_i^2 = 1$ for $x_i \in \{-1,+1\}$, and therefore the diagonal elements of $\Theta$ do not offer any degrees of freedom.

\noindent \paragraph{Poisson graphical models} Here $\mathcal{X} = \mathbb{Z}_+$, the ancillary statistic is $h(x_1,\dots,x_d) = \prod_{i=1}^d \tfrac{1}{x_i!}$, and $\nu$ is the counting measure.  To ensure that each $x_i | x_{\backslash i}$ is distributed as a Poisson random variable, each $\Theta_{i,i} = 0$ for $i=1,\dots,d$.  The set of valid parameters for which the associated distribution is normalizable is given by:
\begin{equation} \label{eq:poissonvalid}
\mathcal{F}_{\mathrm{poisson}} = \{(\alpha,\Theta) \in \R^d \times \Sym^d ~|~ \Theta_{i,j} \geq 0, i,j=1,\dots,d; ~ \Theta_{i,i} = 0, i=1,\dots,d\}.
\end{equation}

\noindent \paragraph{Exponential graphical models} Here $\mathcal{X} = \R_+$, the ancillary statistic is $h \equiv 1$, and $\nu$ is the Lebesgue measure.  To ensure that each $x_i | x_{\backslash i}$ is distributed as an exponential random variable, each $\Theta_{i,i} = 0$ for $i=1,\dots,d$.  The set of valid parameters for which the associated distribution is normalizable is given by:
\begin{equation} \label{eq:exponentialvalid}
\begin{aligned}
\mathcal{F}_{\mathrm{exponential}} = \{(\alpha,\Theta) \in \R^d \times \Sym^d ~|~ & \alpha_i < 0, i=1,\dots,d; ~ \Theta_{i,j} \geq 0, i,j=1,\dots,d; \\ & \Theta_{i,i} = 0, i=1,\dots,d\}.
\end{aligned}
\end{equation}

In each of these three cases, the log-partition is intractable to compute; for such situations, a number of convex approximations of the partition function that are tractable to compute are available in the literature (see \cite{wrjordan} and the references therein), and these may be employed as surrogates \eqref{eq:clvgm} to obtain computationally efficient convex relaxations.  In our numerical experiments in Sections~\ref{section:phase} and \ref{sec:real}, we use the following pseudo-likelihood approximation due to Besag \cite{Besag}:
\begin{equation} \label{eq:besagpseudo}
f(x_1,\dots,x_d | z; \alpha, \Theta, B) \approx \prod_{i=1}^d f(x_i | x_{\backslash i}, z; \alpha, \Theta, B).
\end{equation}
For exponential family distributions of the form \eqref{eq:expfam}, this approximation replaces partition functions associated to $d$-dimensional distributions that are potentially expensive to compute by a collection of $d$ one-dimensional partition functions.  In the three particular examples above, the diagonal elements of $\Theta$ are zero.  Further, the ancillary statistic $h$ is a product of the form $h(x_1,\dots,x_d) = \prod_{i=1}^d \underline{h}(x_i)$, with $\underline{h} = 1$ for the Bernoulli and exponential case and $\underline{h}(y) = \tfrac{1}{y!}$ for the Poisson case.  Consequently, we obtain the following expression:
\begin{equation} \label{eq:besagexpfam}
\begin{aligned}
\prod_{i=1}^d f(x_i | x_{\backslash i}, z; \alpha, \Theta, B) &= \exp\left\{(\alpha+Bz)' x - x' \Theta x - \sum_{i=1}^d \rho\left(\alpha_i + (Bz)_i - \Theta_{i,\backslash i} x_{\backslash i} \right) \right\} h(x) \\ \rho(u) &\triangleq \int_{\mathcal{X}} \exp\left\{uy \right\} \underline{h}(y) d\underline{\nu}(y).
\end{aligned}
\end{equation}
Here $\underline{\nu}$ represents either the one-dimensional counting measure (Bernoulli, Poisson) or the Lebesgue measure on $\R$ (exponential).  Thus, each of the $d$ terms $\rho(\alpha_i + (Bz)_i - \Theta_{i,\backslash i} x_{\backslash i})$ corresponding to the normalization for each $f(x_i | x_{\backslash i})$ entails a one-dimensional integral/sum, and in each of the three examples above, the function $\rho$ is expressible in closed form.  With this approximation, we obtain the following regularized conditional pseudo-likelihood optimization problem given data $\{x^{\scriptscriptstyle (k)}\}_{k=1}^n \subset \mathcal{X}^d$ and user-specified regularization parameters $\lambda, \gamma \geq 0$:
\begin{equation} 
\begin{aligned}
(\hat{\alpha}, \hat{\Theta}, \hat{L}) = \mathop{\arg\min}_{\substack{\alpha \in \R^d, \Theta \in \Sym^d \\ L \in \R^{d \times n}}} ~ &\frac{1}{n} \sum_{k=1}^n \left[ \left( \sum_{i=1}^d \rho\left(\alpha_i + L^{\scriptscriptstyle (k)}_i - \Theta_{i,\backslash i} x^{\scriptscriptstyle (k)}_{\backslash i} \right) \right) - (\alpha + L^{\scriptscriptstyle (k)})' x^{\scriptscriptstyle (k)} + {x^{\scriptscriptstyle (k)}}' \Theta x^{\scriptscriptstyle (k)} \right] \\ &+ \lambda \|\Theta\|_{\ell_1} + \gamma \|L\|_\star \\ \mathrm{s.t.} ~ & (\alpha + L^{\scriptscriptstyle (k)}, \Theta) \in \mathcal{F}, ~~~ k=1,\dots,n.
\end{aligned}
\label{eq:pseudoclvgm}
\end{equation}
We can specialize this convex relaxation to each of the three examples described above with the corresponding choice of valid parameters in the constraint and the following one-dimensional log-partition functions in the objective:
\begin{equation} \label{eq:nongaussianlogpart}
\rho_{\mathrm{ising}}(u) = \log \cosh(u) ~~~~~ \rho_{\mathrm{poisson}}(u) = \exp(u) ~~~~~ \rho_{\mathrm{exponential}}(u) = -\log(-u).
\end{equation}

If we do not account for the confounding effects of latent variables and set $L = 0$, then we recover a ``coupled'' analog of the neighborhood selection approaches of \cite{Meinhausen,Ravikumar,Genevera}, which identify the neighborhood of each node in the graph one at a time by solving $d$ uncoupled $\ell_1$-regularized regression problems.  One (relatively minor) issue with solving $d$ uncoupled problems is that one subsequently needs to reconcile the solutions to obtain a coherent global model over all the variables, although there are several ways to accomplish this \cite{Hofling,Meinhausen,Ravikumar,Genevera}.  A more significant issue with solving $d$ uncoupled neighborhood selection problems is that it is not clear how to adapt that method to account for the effects of latent variables.  In particular, latent variables can simultaneously influence all the observed variables, which necessitates an approach that jointly estimates the (local) neighborhoods of all the nodes at the same time while also teasing apart the (global) effects of the latent variables, as in \eqref{eq:pseudoclvgm}.  In another direction, if we set $\Theta = 0$ in \eqref{eq:pseudoclvgm} we obtain a convex relaxation for the exponential family PCA problem \cite{epca}. 
%Since the decision variables in Equation \eqref{eq:pseudoclvgm} are coupled in an additive manner, we employ alternating direction method of multipliers (ADMM) to obtain provably optimal solutions \cite{ADMM}.

\subsection{Empirical demonstrations for Gaussian and non-Gaussian models}
\label{section:phase}
\vspace{0.1in}
We evaluate next the empirical performance of the relaxation for the Gaussian case \eqref{eq:gaussianreduced} and the relaxation for the non-Gaussian case \eqref{eq:pseudoclvgm} for fitting Ising, Poisson, and exponential graphical models when confounded by latent variables. The following are common elements of the setup for each distributional setting: we consider a collection of $d = 60$ observed variables whose distribution conditioned on some latent variables is given by a Gaussian, Ising, Poisson, or exponential graphical model; the distribution of the latent variables is specified later in each case.  We generate two types of graphs with a corresponding $\Theta \in \mathbb{S}^{60}$: a cycle graph and a Erd\"os R\'eyni graph with edge probabilities $0.02$. We vary the number of latent variables $r = \{1,2,3\}$ and generate the matrix $B \in \mathbb{R}^{60 \times r}$ so that the coherence\footnote{The coherence of a subspace $\mathcal{S} \subset \R^d$ measures how well $\mathcal{S}$ is aligned with the standard basis vectors in $\R^d$; it is equal to  $\max_{i=1,\dots,d} ~ \| \mathcal{P}_{\mathcal{S}}(e^{(i)}) \|_{\ell_2}$, where $e^{(i)}$ is the $i$'th standard basis vector.  The coherence of $\mathcal{S}$ lies in the range $[\sqrt{\mathrm{dim}(\mathcal{S}) / d},1]$, and this parameter commonly arises in the characterization of statistical identifiability as well as in the analysis of convex relaxations in sparse/low-rank recovery problems.} of its column-space is approximately $1.2{r}/{d}$. The singular values of $B$ vary for each distribution and are described below. Finally, unless otherwise specified, we set $\alpha \in \mathbb{R}^{60}$  to be the identically zero vector, and the nonzero off-diagonal entries of $\Theta$ to be $0.4$. 

{\emph{Gaussian setup}}: The distribution of the observed variables conditioned on zero-mean Ising latent variables is a Gaussian graphical model.  The diagonal entries of $\Theta$ are set to $1$.  The singular values of $B$ are chosen to be $\{0.72\}$,$\{0.7,0.7\}$ and $\{0.68,0.68,0.68\}$ for $r = 1,2,3$ latent variables, respectively. 

{\emph{Ising setup}}: The distribution of the observed variables conditioned on independent normally distributed hidden variables is an Ising graphical model.  The singular values of $B$ are $\{0.72\}$,$\{0.7,0.7\}$ and $\{0.68,0.68,0.68\}$ for $r = 1,2,3$ latent variables, respectively. 

{\emph{Poisson setup}}: The distribution of observed conditioned on independent and identically distributed zero-mean Ising hidden variables is a Poisson graphical model.  The singular values of $B$ are chosen to be $\{2\},\{1.95,1.95\},\{1.9,1.9,1.9\}$ for $r = 1,2,3$ latent variables, respectively. 

{\emph{Exponential setup}}: The distribution of the observed variables conditioned on independent and identically distributed mean-$1$ exponential hidden variables is an exponential graphical model. Due to the parameter restriction $\mathcal{F}_\text{exponential}$ in \eqref{eq:exponentialvalid}, the entries of $\Theta$ must be non-negative, entries of $B$ must be non-positive, and $\alpha$ must consist of negative entries.  We set the edge weights (i.e., non-zero entries of $\Theta$) to be $1$ and the singular values of $B$ are chosen to be $\{2\},\{1.95,1.95\},\{1.9,1.9,1.9\}$ for $r = 1,2,3$ latent variables, respectively.  Finally, we set all the entries of $\alpha$ to be equal to $-1$. 

For each problem setting, we generate observations via Gibbs sampling to obtain training data $\{x^{(i)}\}_{i = 1}^n \subseteq \mathbb{R}^d$. We supply the data to the estimator \eqref{eq:gaussianreduced} for the Gaussian model and to the estimator \eqref{eq:pseudoclvgm} for non-Gaussian models (with $\rho$ selected suitably). The regularization parameters $\lambda,\gamma$ are chosen with the scaling $\lambda = c_1\sqrt{\frac{d}{n}}$ and $\gamma = c_2\frac{\sqrt{d}}{n}$, for constants $c_1,c_2$.  We evaluate the probability (computed over ten independent trials) that the estimated model correctly identifies the graphical structure as well as the number of latent variables. Figure~\ref{fig:consist} displays the empirical consistency results for all problem settings. We observe that given sufficient sample size, the estimators \eqref{eq:gaussianreduced} and \eqref{eq:pseudoclvgm} are successful at correctly identifying the model structure. 
\FloatBarrier
\begin{figure}[ht!]
\centering
\subfigure[\small{Gaussian: cycle}]{
\includegraphics[scale = 0.20]{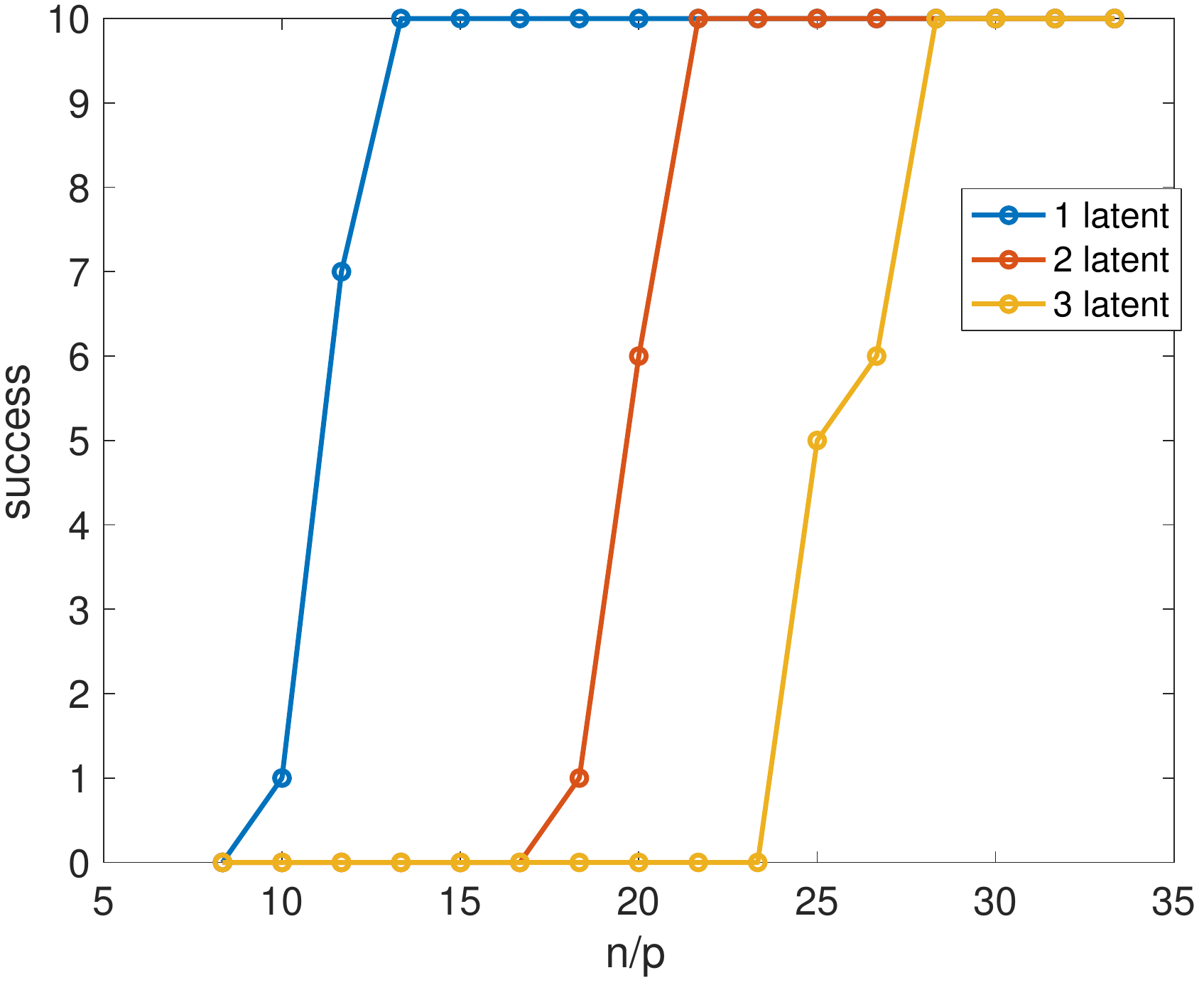}}
\subfigure[\small{Gaussian: E-R}]{
\includegraphics[scale = 0.20]{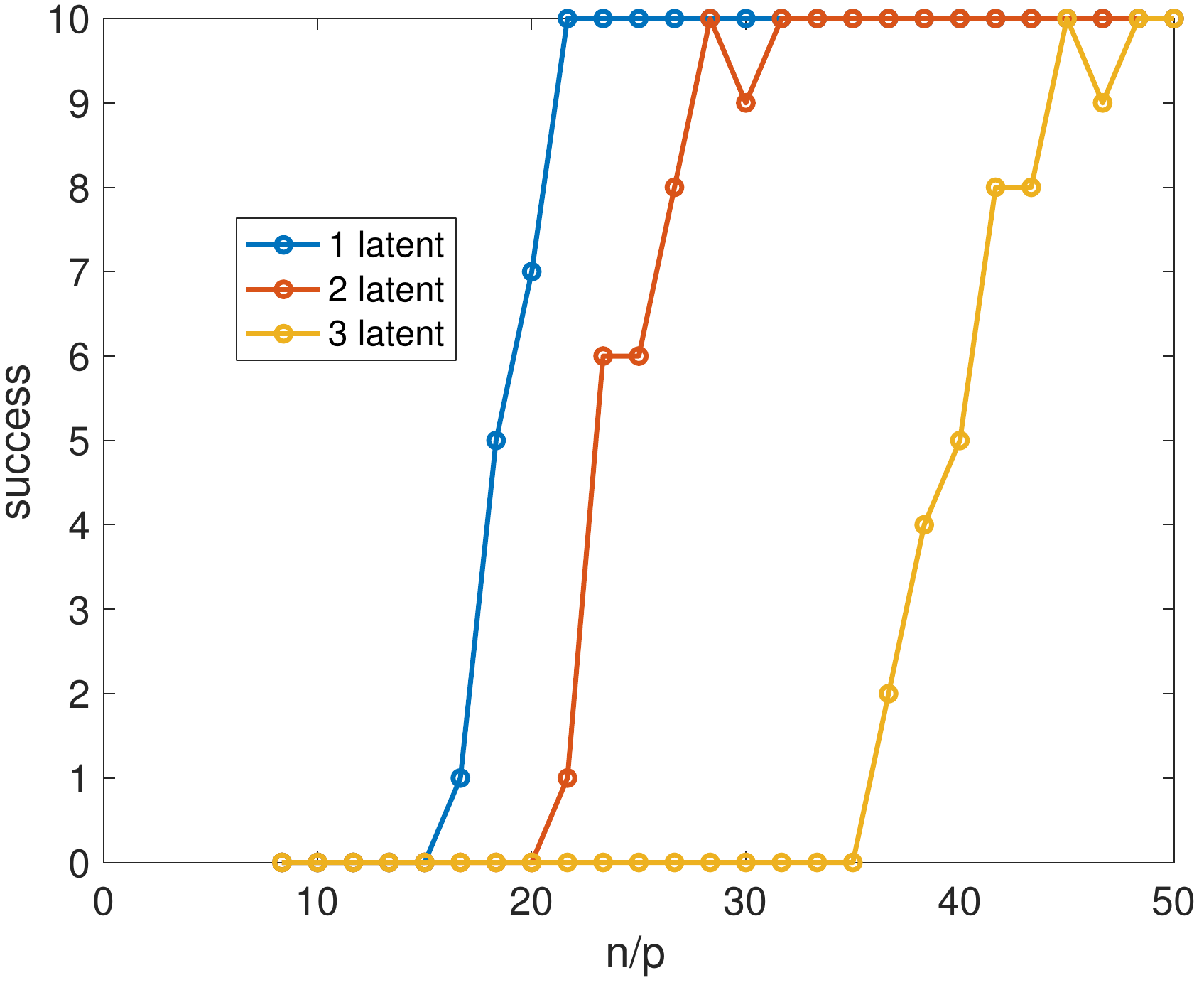}}
\subfigure[Ising: cycle]{
\includegraphics[scale = 0.20]{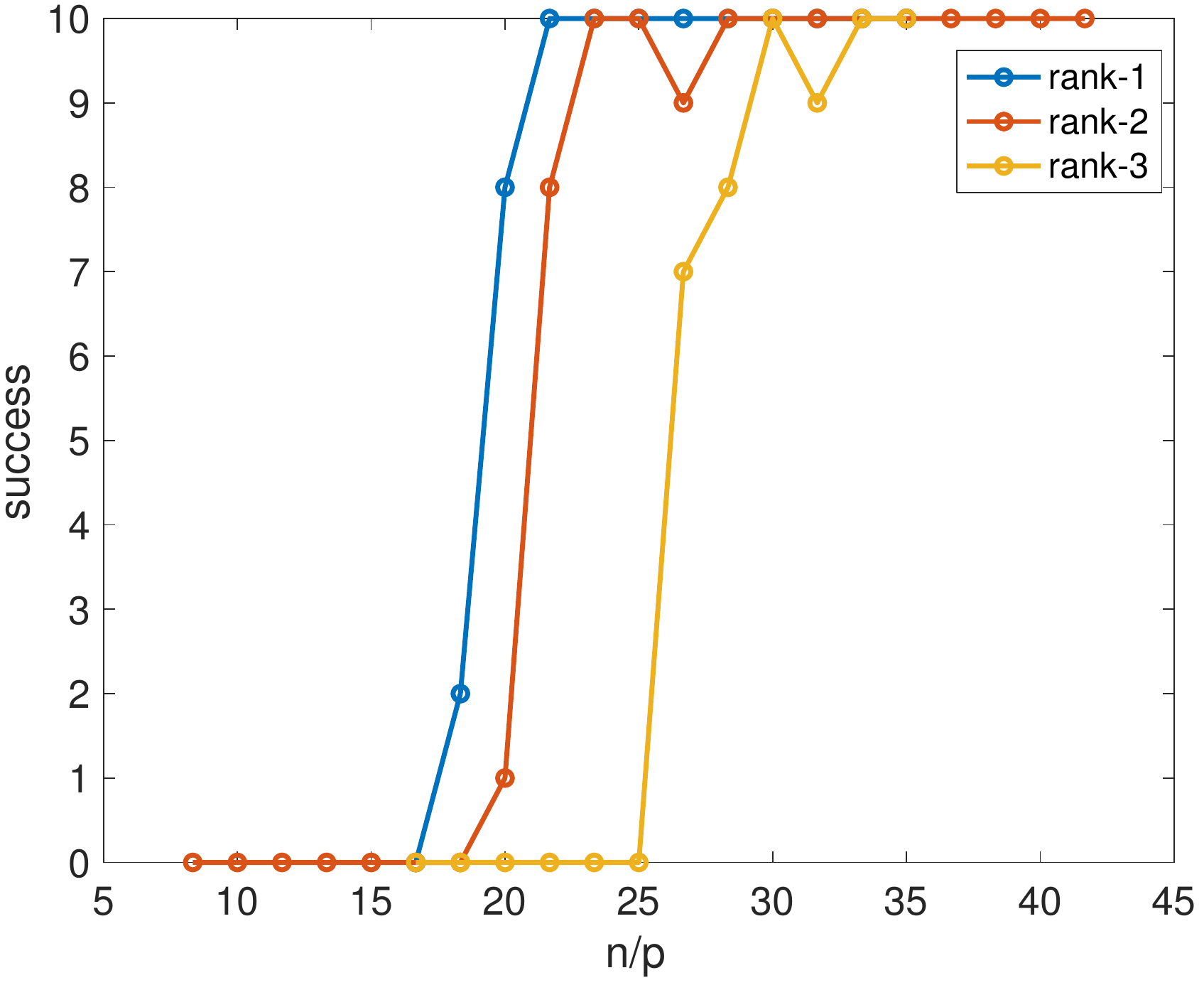}}
\subfigure[Ising: E-R]{
\includegraphics[scale = 0.20]{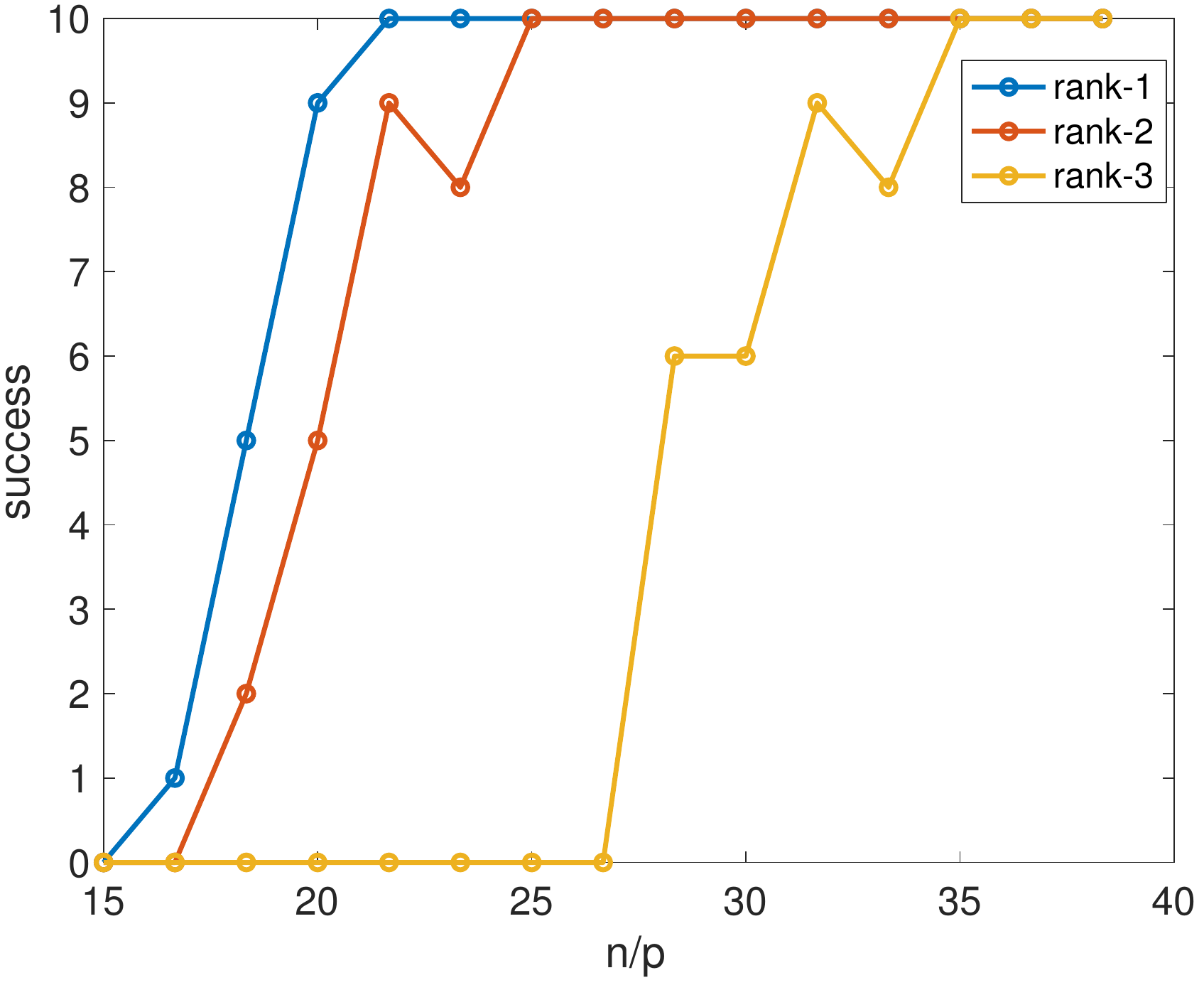}}
\subfigure[Poisson: cycle]{
\includegraphics[scale = 0.20]{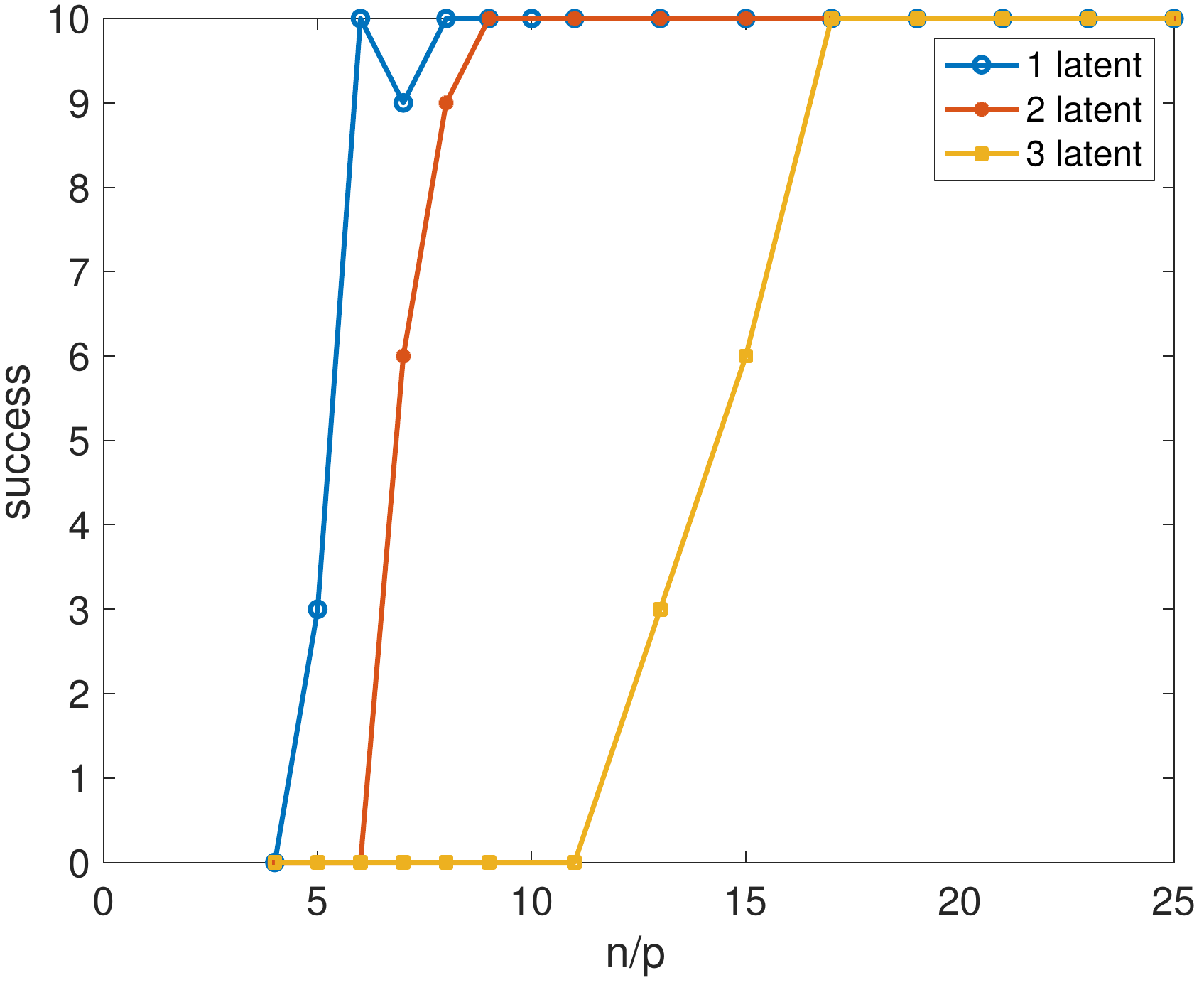}}
\subfigure[Poisson: E-R]{
\includegraphics[scale = 0.20]{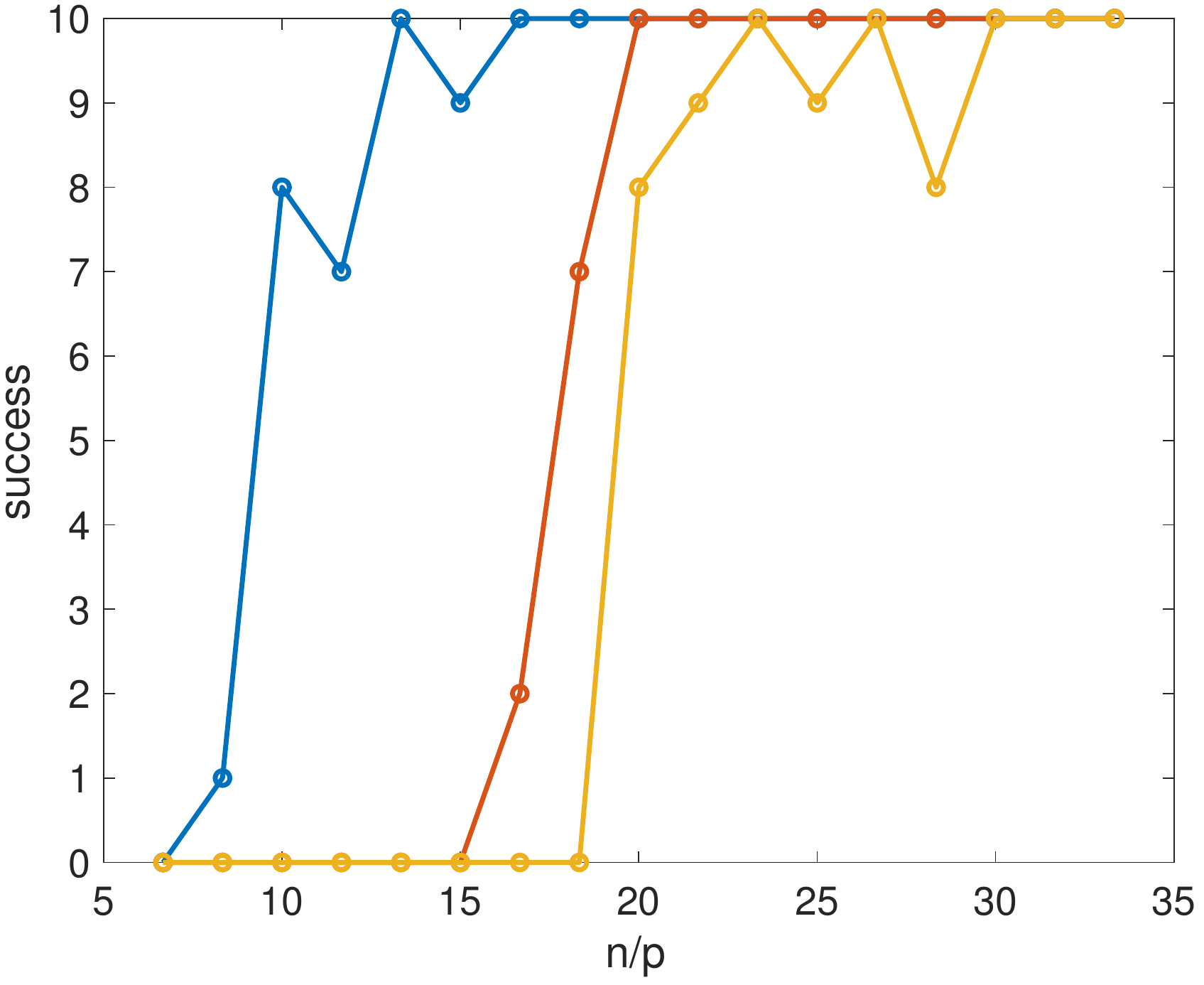}}
\subfigure[Exponential: cycle]{
\includegraphics[scale = 0.34]{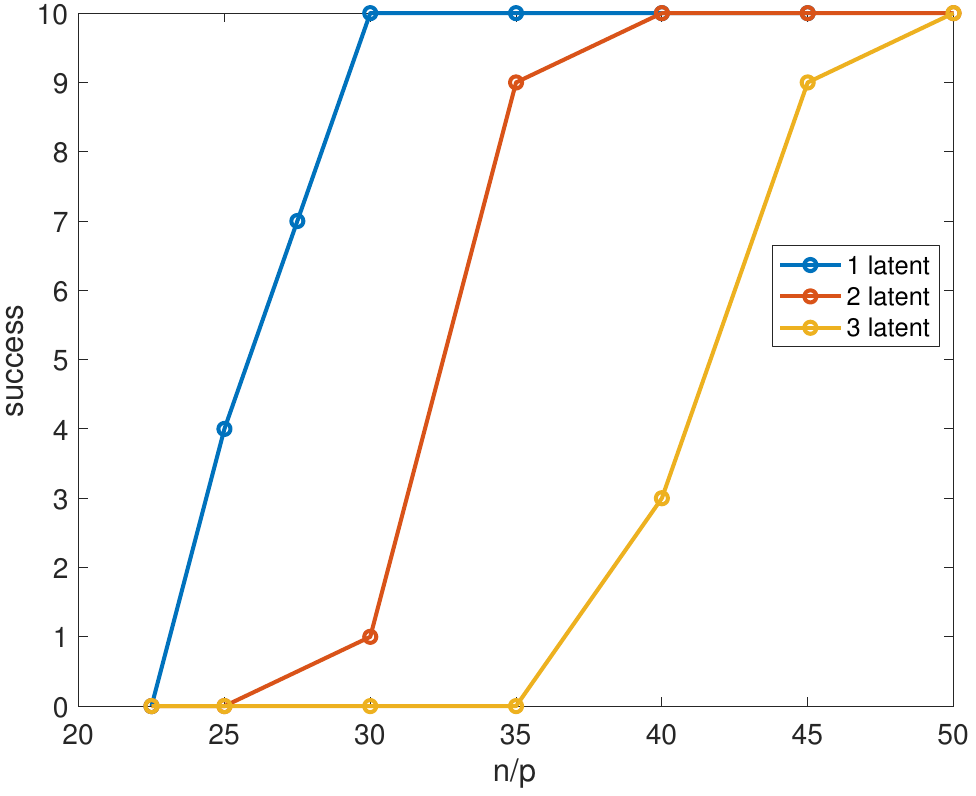}}
\subfigure[Exponential: E-R]{
\includegraphics[scale = 0.34]{exponential_consist.pdf}}
\caption{Probability of correcting identifying the population graphical model structure and the number of latent variables (computed empirically across 10 trials) with the estimators \eqref{eq:gaussianequiv} and \eqref{eq:pseudoclvgm} for cycle and Erd\"os-R\'{e}nyi graph and $r = 1,2,3$ latent variables.}
\label{fig:consist}
\end{figure}
\FloatBarrier
%\textit{Impact of graph degree and latent incoherence}:The accuracy of the estimator Equation \eqref{eq:pseudoclvgm} is dependent on the degree of the population graph and the incoherence of the latent effects. This phenomena has been observed in previous works for graphical modeling in the presence of latent variables [CITE]. In this experiment, we aim to understand the extent to which the degree of the graph 

\section{Model Selection}
\label{sec:model_sel}
The selection of the regularization parameters $\lambda,\gamma$ in \eqref{eq:clvgm} (as well as its specializations and approximations \eqref{eq:gaussianclvgm} and \eqref{eq:pseudoclvgm}) is an important consideration in obtaining a useful model.  Standard approaches such as cross-validation tend to yield overly complex models that can overfit to the data \cite{Meinhausen} (in our context, such models correspond to those in which the graph structure is dense and the number of latent variables is large).  To address this issue, several methods have been proposed in the literature based on a notion of \emph{stability} \cite{Wasserman,stability}, in which a model selection procedure is applied to subsamples of a dataset and the variability of the resulting solutions over the subsamples governs the identification of a suitable regularization parameter or model structure.  By combining the ideas in \cite{Wasserman,stability}, we present a model selection technique that is suited to the context of the present paper in Section~\ref{section:stability}.  We demonstrate the utility of these techniques via numerical experiments in Section~\ref{section:experimental_model}.

%in the context of sparse model selection

\subsection{Model Selection via Stability}
\label{section:stability}
To assess the variability of the structure of the selected models over subsamples, we need an appropriate method to aggregate the models selected over subsamples.  Concretely, let $\mathcal{D} = \{x^{(k)}\}_{k=1}^n$ be the given dataset and consider $B$ subsamples $\mathcal{D}^{(l)} \subset \mathcal{D}, ~ l=1,\dots,B$ with $|\mathcal{D}^{(l)}| = |\mathcal{D}| / 2$.  Fix regularization parameters $\lambda,\gamma$, and let $(\hat{\Theta}_{\lambda, \gamma}^{(l)}, \hat{L}_{\lambda, \gamma} ^{(l)}), ~l=1,\dots,B$ represent the optimal solutions obtained from \eqref{eq:gaussianclvgm} or \eqref{eq:pseudoclvgm} of the $B$ subsamples (for the purposes of model structure aggregation, $\alpha$ plays no role).  To represent the variability in the graphical model structure across these optimal solutions, form a diagonal matrix $\Proj_{\lambda, \gamma}^{\mathrm{graph}} \in \Sym^{{d \choose 2}}$ as follows:
\begin{equation*}
\left(\Proj_{\lambda, \gamma}^{\mathrm{graph}} \right)_{e,e} = \frac{1}{B} \sum_{l=1}^B \mathbb{I}\left(e \in \mathrm{support}(\hat{\Theta}_{\lambda, \gamma}^{(l)}) \right).
\end{equation*}
Here $e$ ranges over the ${d \choose 2}$ edges and $\mathbb{I}$ denotes the indicator function that equal $1$ if its argument is true and $0$ otherwise.  In words, the diagonal entries of $\Proj_{\lambda, \gamma}^{\mathrm{graph}}$ lie in $[0,1]$ and they encode the frequencies of the edges appearing in the selected models aggregated over the subsamples.  The methods presented in \cite{Wasserman,stability} may be described in terms of this matrix, and they were applicable to discrete model selection problems such as graph estimation and variable selection.  These ideas were extended recently to low-rank estimation problems based on a geometric reformulation of model selection \cite{false_discovery}, with a key ingredient being a suitable generalization of the aggregate matrix $\Proj_{\lambda, \gamma}^{\mathrm{graph}}$.  Specifically, for each $\hat{L}_{\lambda, \gamma}^{(l)}$ let $\Proj_{\lambda,\gamma}^{(l)} \subset \mathbb{S}^d$ denote the projection operator onto the column-space of $\hat{L}_{\lambda, \gamma}^{(l)}$.  With this notation, the variability in the structure underlying the low-rank estimates across subsamples is specified by the following average projection map:
\begin{equation*}
\Proj_{\lambda,\gamma}^{\mathrm{latent}} = \frac{1}{B} \sum_{l=1}^B \Proj_{\lambda,\gamma}^{(l)}
\end{equation*}
The eigenvalues of $\Proj_{\lambda, \gamma}^{\mathrm{latent}}$ lie in the range $[0,1]$.  We describe next the three main steps of our model selection approach.

%Our methods are based on an appropriate notion of stability in which we fix $\lambda,\gamma$ and consider the stability of the graph structure (the nonzero pattern of $\Theta$) and the influence of the latent variables (the column-space of $L$) ...

\paragraph{Stage 1: Identifying Regularization Parameters}  The first step is to select appropriate values of $\lambda,\gamma$.  Building on the insights of \cite{Wasserman}, let $\pi^{\mathrm{graph}}_{\lambda,\gamma} = \tfrac{1}{{d \choose 2}}[\mathrm{trace}(\Proj_{\lambda,\gamma}^{\mathrm{graph}}) - \mathrm{trace}(\Proj_{\lambda,\gamma}^{\mathrm{graph}})^2]$ and $\pi^{\mathrm{latent}}_{\lambda,\gamma} = \tfrac{1}{d} [\mathrm{trace}(\Proj_{\lambda,\gamma}^{\mathrm{latent}}) - \mathrm{trace}(\Proj_{\lambda,\gamma}^{\mathrm{latent}})^2]$ denote the total variabilities in the graphical and latent components, respectively.  These parameters lie in the range $[0,1]$, and they are small when the graph structure and the latent subspace are stable across subsamples.  For sufficiently large values of $\lambda,\gamma$, the graph structure is completely disconnected and the latent subspace is zero-dimensional; correspondingly, $\pi^{\mathrm{graph}}_{\lambda,\gamma}$ and $\pi^{\mathrm{latent}}_{\lambda,\gamma}$ are both zero.  As $\lambda,\gamma$ are gradually decreased, more edges and higher-dimensional subspaces are progressively included in the recovered graph structure and latent components, and the values of $\pi^{\mathrm{graph}}_{\lambda,\gamma}$ and $\pi^{\mathrm{latent}}_{\lambda,\gamma}$ begin to increase.  When these values reach a desired user-specified threshold, the corresponding $\lambda,\gamma$ are set as the regularization parameters.  (Typical threshold values for $\pi^{\mathrm{graph}}_{\lambda,\gamma}$ and for $\pi^{\mathrm{latent}}_{\lambda,\gamma}$ are $0.025$, thus yielding a total variability of $0.05$ as recommended in \cite{Wasserman}).  This approach provides regularization parameters for which the associated graphical model and latent subspace are sparse/low-dimensional, while exhibiting little overall variability across subsamples.

\paragraph{Stage 2: Identifying Model Structure} Solving \eqref{eq:gaussianclvgm} or \eqref{eq:pseudoclvgm} with the regularization parameters obtained from the preceding step tends to lead to models that have small type-II error (formally \cite{Wasserman} shows that type-II error in graph structure estimation is small under minimal assumptions).  However, to also reduce type-I error it is useful to further restrict the models selected based on a more refined form of stability, as described in \cite{stability,false_discovery}.  Specifically, while the approach of \cite{Wasserman} considers aggregate variability, the methods in \cite{stability,false_discovery} suggest selecting a graphical model structure and a latent subspace that are common to a large proportion of the subsamples.  Concretely, let $\Proj^{\mathrm{graph}}$ and $\Proj^{\mathrm{latent}}$ represent the average projections over subsamples for the values of the regularization parameters chosen from the previous step (we have suppressed the dependence on $\lambda,\gamma$ for notational clarity).  For the graphical model component, we select those edges corresponding to all those elements on the diagonal of $\Proj^{\mathrm{graph}}$ that are above a user-specified threshold $\delta^{\mathrm{graph}} \in [0,1]$; for large values of $\delta^{\mathrm{graph}}$, these correspond to edges that are chosen in a large proportion of the subsamples.  For the latent subspace component, we select the largest-dimensional subspace $\mathcal{C} \in \R^d$ such that $\sigma_{\mathrm{min}}(\Proj_{\mathcal{C}} \Proj^{\mathrm{latent}} \Proj_{\mathcal{C}}) \geq \delta^{\mathrm{latent}}$.  Here $\delta^{\mathrm{latent}} \in [0,1]$ is again user-specified, $\Proj_{\mathcal{C}} \in \mathbb{S}^d$ denotes the projection onto $\mathcal{C}$, and $\sigma_{\mathrm{min}}(\Proj_{\mathcal{C}} \Proj^{\mathrm{latent}} \Proj_{\mathcal{C}})$ is the smallest singular value of the operator $\Proj_{\mathcal{C}} \Proj^{\mathrm{latent}} \Proj_{\mathcal{C}}$ viewed as a self-adjoint map on $\mathcal{C}$.  Selecting such a subspace may be accomplished by a singular value decomposition of $\Proj^{\mathrm{latent}}$, and for large values of $\delta^{\mathrm{latent}}$, the selected subspace is one that well-aligned with the subspaces that are chosen in a large proportion of the subsamples.  (A typical recommended value for both $\delta^{\mathrm{graph}}, \delta^{\mathrm{latent}}$ is $0.7$, as suggested in \cite{stability,false_discovery}).  As shown in \cite{stability} for sparse models and in \cite{false_discovery} for low-rank models, such stability-based approaches yield models with small type-I error.

\paragraph{Stage 3: Identifying Model Parameters} The output of the preceding step is a stable subset of edges $\mathcal{E}$ for the graphical model and a stable column-space $\mathcal{C}$ for the latent component.  With these in hand, we solve either \eqref{eq:gaussianclvgm} or \eqref{eq:pseudoclvgm} with two modifications.  First, we add the constraint that $\Theta$ must lie in the subspace of matrices in which the entries indexed by $\mathcal{E}^c$ equal zero and the constraint that $L$ must lie in the subspace of matrices in which each column lies in $\mathcal{C}$.  Second, we set the regularization parameters $\lambda = \gamma = 0$ as these are no longer required to obtain low-complexity models.  Even with these modifications \eqref{eq:gaussianclvgm} and \eqref{eq:pseudoclvgm} continue to be tractable convex optimization problems.

\subsection{Experimental Demonstration}
\label{section:experimental_model}
\vspace{0.05in}
We provide empirical demonstration of the utility of the model selection method presented above in terms of the false discovery rate (FDR) and true positive rate (PWR) of the estimated graph structure; the FDR is the expected ratio of the number of estimated edges that are not in the true underlying graph over the total number of estimated edges and the PWR is expected ratio of the number of estimated edges that are in the true underlying graph over the total number of estimated edges.

We consider the setting where the conditional graphical model of $50$ observed variables conditioned on two independent normally distributed latent variables is an Ising model, with the population graphical structure being an Erd\"os-R\'enyi graph with edge selection probability $0.02$ and edge weights $0.4$. The coefficient matrix $B \in \R^{50 \times 2}$ is a random partial orthogonal matrix sampled uniformly from the Haar measure.  We obtain observations (via Gibbs sampling) and compute the FDR and PWR over $10$ trials based on the above problem setup using the estimator \eqref{eq:pseudoclvgm} with $\rho = \rho_{\mathrm{ising}}$. Figure~\ref{fig:model_selection_ising} demonstrates the graph recovery performance after employing the first stage of our model selection approach as well as combining both the first and second stages. Notice that for moderate $n$, the first stage of the algorithm yields a graphical structure with $\text{PWR} \approx 1$ but also high $\text{FDR}$ (i.e. many false positives). After the second stage, we substantially reduce FDR without much loss in power.  These results provide empirical support for the utility of our two-stage model selection method, and in particular the fact that combing both stages yields graphical models that have small Type-I error as well as small Type-II error.
%\begin{SCfigure}
\begin{figure}[ht!]
\centering
\includegraphics[scale = 0.3]{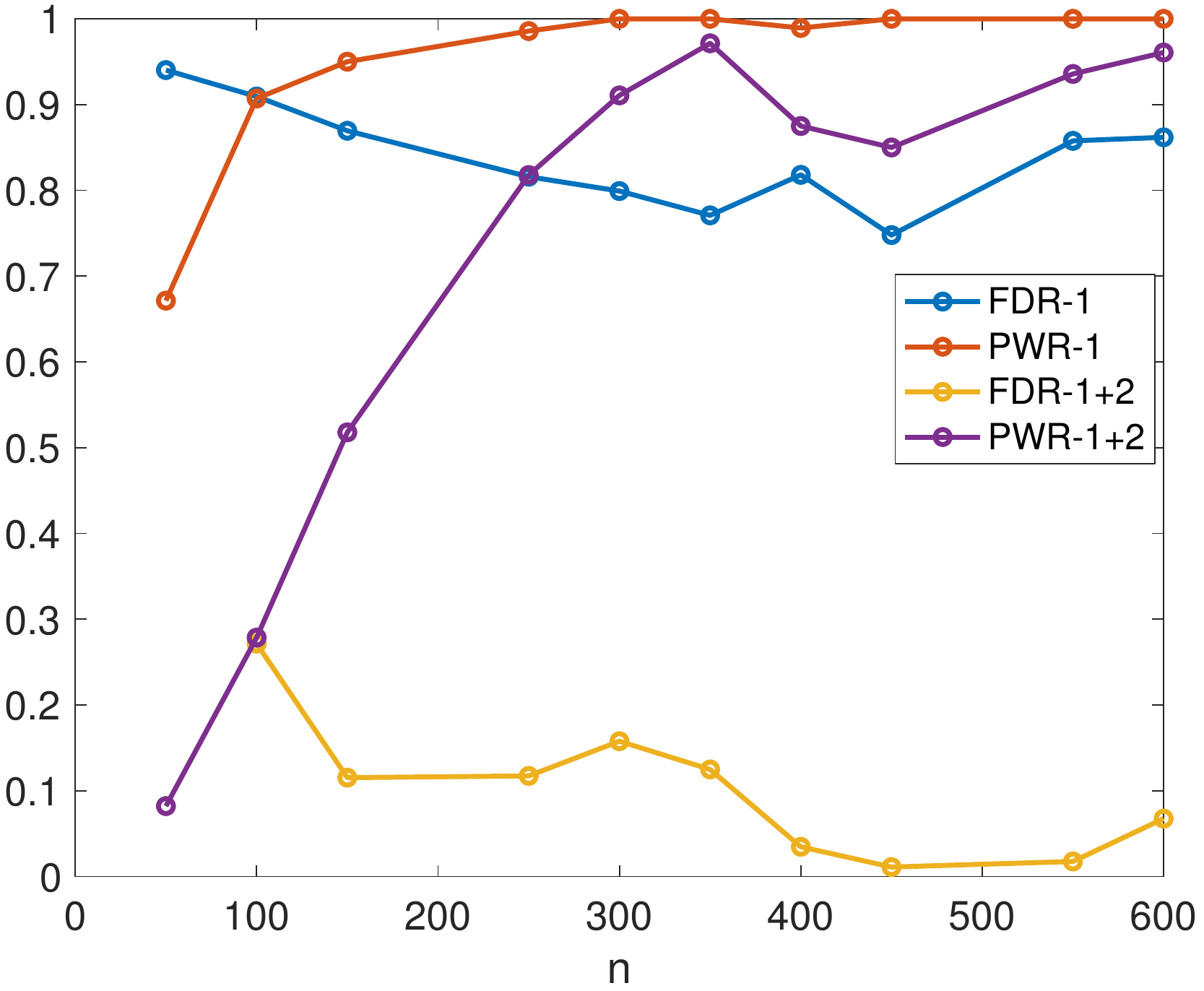}
\caption{False discovery rate (FDR) and true positive rate (PWR) as a function of $n$ after the first stage of the proposed model selection procedure, denoted by `FDR-1' and `PWR-1', and after the second stage of the model selection procedure, denoted by `FDR-1+2' and `PWR-1+2'.}
\label{fig:model_selection_ising}
%\end{SCfigure}
\end{figure}

\section{Experiments with Real Data}
\label{sec:real}
In this section, we demonstrate the utility of our latent-variable graphical modeling framework on US Senate voting records data, mi-RNA sequence data, and S$\&$P-500 stock data.  We provide comparisons between the graphical structure obtained via our approach and graphical models that do not incorporate latent variables.  In addition to examining the difference between the graphical structure of the approach with latent variables to the one without, we also provide quantitative comparison of their prediction performance.  Specifically, let $\mathcal{D}_\text{train} = \{x^{(k)}\}_{k = 1}^{n_\text{train}}$ and $\mathcal{D}_\text{test} = \{x^{(i)}\}_{k = 1}^{n_\text{train}}$ denote training and test datasets, respectively.  Using \eqref{eq:gaussianclvgm} and \eqref{eq:pseudoclvgm}, let $(\hat{\alpha}_\text{latent}, \hat{\Theta}_{\text{latent}}, \hat{L}_\text{latent})$ denote the estimated parameters based on our approach that incorporates latent variables.  Next, let $(\hat{\alpha}_\text{no-latent}, \hat{\Theta}_{\text{no-latent}})$ denote the estimated parameters based on fixing $L = 0$ so that latent variables are not incorporated.  In obtaining these models, we employ the model selection technique described in Section~\ref{sec:model_sel} with the various thresholds chosen as stated in the corresponding stages; the one distinction is that when we fix $L=0$ to obtain a graphical model that does not incorporate latent variables, we fix the variability threshold for $\pi^{\text{graph}}_\lambda$ to be $0.05$.  We evaluate the prediction performance of the two models by comparing negative log (pseudo-)likelihood values on the test data $\mathcal{D}_\text{test}$.  These values are obtained by solving unregularized (i.e., $\lambda = \gamma = 0$) and suitably constrained versions of \eqref{eq:gaussianclvgm} and \eqref{eq:pseudoclvgm}.  In particular, for the setting with latent variables, we consider the optimal values of these problems with the additional constraints $\alpha = \hat{\alpha}_\text{latent}, \Theta = \hat{\Theta}_{\text{latent}}, \text{col-space}(L) \subseteq \text{col-space}(\hat{L}_\text{latent})$, and for the setting with no latent variables we consider the optimal values with the constraints $\alpha = \hat{\alpha}_\text{no-latent}, \Theta = \hat{\Theta}_{\text{no-latent}}, L = 0$.

%We focus on the relaxation \eqref{eq:gaussianclvgm} for the non-Gaussian settings, although the description is applicable to the Gaussian relaxation \eqref{eq:gaussianclvgm}.

\subsection{Senate voting records data}
We apply our latent-variable modeling framework to the 109th Senate voting record dataset.   The dataset was obtained from the website of the US Congress (\url{http://www.
senate.gov}). It contains the voting records of the 100 senators  -- 55 Republicans, 44 Democrats, and one Independent -- of the 109th congress (January 3, 2005 — January 3, 2007) on 645 bills on which the Senate voted.  The votes are recorded as $+1$ for ``yes” and $-1$ for ``no”. The data contains missing votes as some senators abstained on a small number of bills. The  missing values (missed votes) for each senator were imputed with the majority vote of that senator’s party on that particular bill and the missing votes of the Independent Senator Jeffords were imputed with the Democratic majority vote (because he caucused with the Democrats). Finally, we exclude bills where the ``yes/no" proportion fell outside the interval $[0.3,0.7]$. This results in $n = 479$ votes across $d = 100$ senators to yield a dataset of $\mathcal{D} = \{x^{(k)}\}_{k = 1}^{479} \subset \{-1,+1\}^{100}$. We take the first $383$ samples as training set $\mathcal{D}_\text{train}$ and the remaining $96$ samples as test data $\mathcal{D}_\text{test}$. 

We fit Ising models with and without latent variables to this dataset.  We obtain a latent-variable graphical model with $r = 8$ latent variables and a conditional graphical model with the number of edges equal to $6\%$ of the total number of pairs of variables.  In contrast, the model without latent variables is given by a graphical model with edge density $\approx 24\%$.  The edge weights of the graph structure in the latent-variable graphical model are shown in the bottom half of Figure~\ref{fig:voter} and the edge weights of the graphical model without latent variables are shown in the top half of Figure~\ref{fig:voter}.  The majority of the interactions in the estimated edges in the graphical models occur between individuals in the same party.  The incorporation of latent variables substantially reduces the number of edges as many confounding dependencies are removed. Examining the positive interactions in the model with latent variables, the strongest edge among the Democrats is between senators Pryor-Lautenburg and among the Republicans is between senators Robert-Inhofe.  We observe that conditioning on the latent variables induces some negative dependencies between senators in the same party, notably Pryor-Baucus $\&$ Reed-Levin among the Democrats and Enzi-Coburn $\&$ Sessions-Cornyn among the Republicans.  Finally, the negative log (pseudo-)likelihoods evaluated on the test data (in the manner described above) yield values of $38.3$ without latent variables and $33.8$ with latent variables, suggesting that our approach which incorporates latent variables more accurately models Senate voting records.

%Lieberman-Kohl, and Stabenow-Bayh and among republicans are between Robert-Inhofe, Lugar-Hagel, and Talent-Robert. Further, the strongest three edges across parties are Leahy-Sessions, Dodd-Coleman, and Landrieu-Isakson.  Examining the negative edges, the strongest among democrats are Pryor-Baucus, Reed-Levin, and Reid-Reed. 

%\begin{SCfigure}
\begin{figure}
\centering
{
\includegraphics[scale = 0.44]{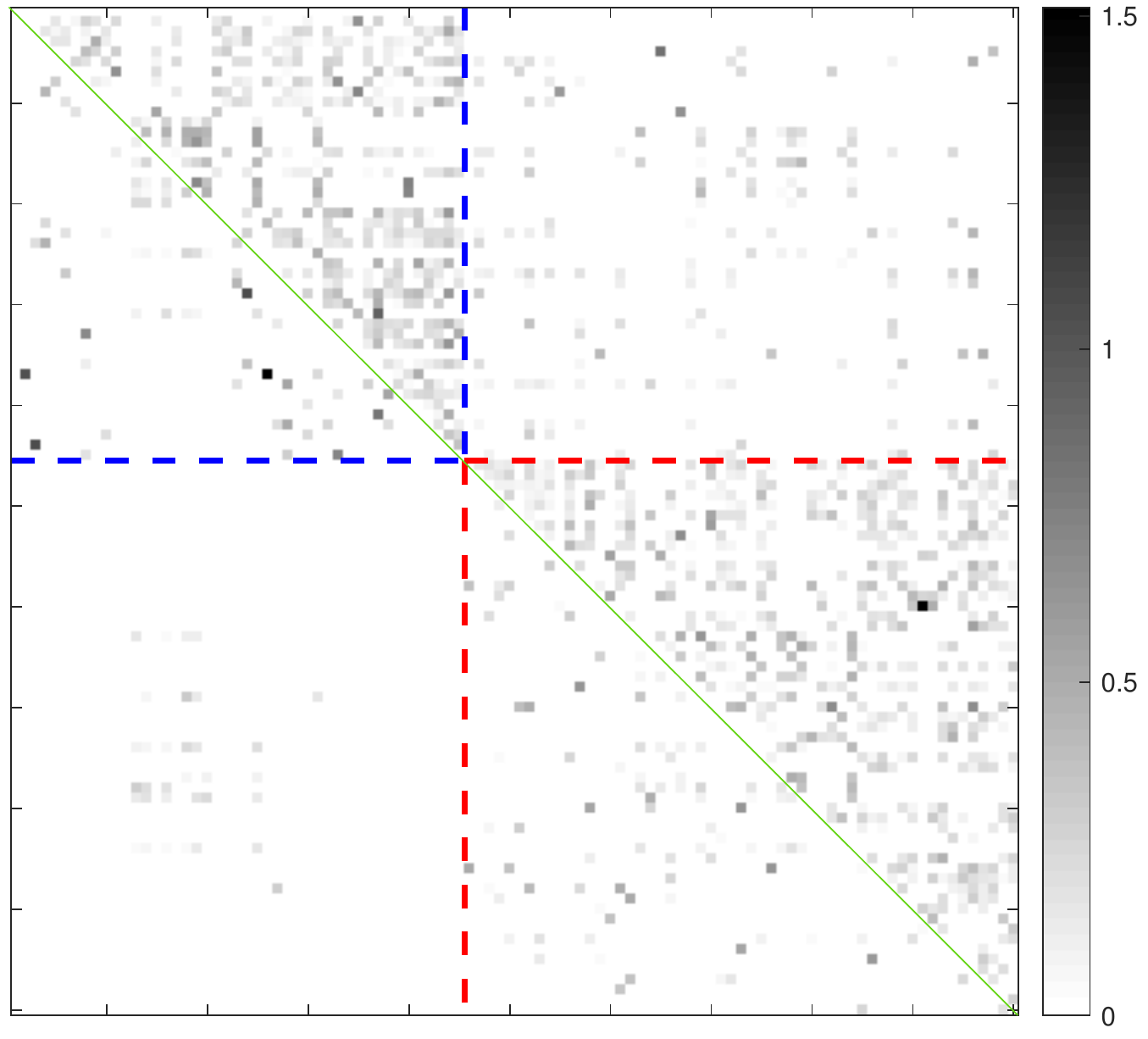}}
\caption{Edges between senator pairs in the graphical model with $8$ latent variables (bottom triangle) compared with those of the no latent variable estimate (tom triangle); the two parties zones are separated by red and blue lines: top left correspond to interactions among Democrats, and bottom right among Republicans; white indicates no edge.}
\label{fig:voter}
\end{figure}
%\end{SCfigure}

\subsection{mi-RNA sequence data}
Next, we demonstrate the utility of our approach in estimating an miRNA inhibitory network for Level III breast cancer miRNA expressions (downloaded from \url{http://tcga-data.nci.nih.gov/tcga/}). The data consist of $262$ miRNAs and $544$ subjects. Of the $262$ miRNAs, we extract $27$ that were considered by \cite{Genevera} (after a hierarchical clustering step) as their interactions are well modeled by negative dependencies (since \eqref{eq:poissonvalid}
only allows for negative dependencies or equivalently $\Theta$ non-negative). The data consisting of the selected $27$ miRNAs were adjusted for possible over-dispersion using a power transform \cite{LiuGen}.  After performing these pre-processing steps, we obtain training data $\mathcal{D}_\text{train} = \{x^{(k)}\}_{k = 1}^{544} \subseteq \mathbb{R}^{27}$ that is well-modeled by a Poisson distribution.

We fit Poisson graphical models with and without latent variables.  We obtain a latent-variable graphical model consisting of $r = 2$ latent variables and a conditional graphical model in which the number of edges is $3\%$ of the total number of pairs of variables.  The graphical model that does not incorporate latent variables has an edge density of $\approx 18\%$.  The corresponding graphs are displayed in the bottom triangle and the top triangle of Figure~\ref{fig:voter}, respectively. We observe that incorporating latent variables in the graphical model removes dependencies between pairs of miRNAs that have similar primary function. Specifically, the strongest edges in the graphical model without latent variables that are not part of the graphical model that incorporates latent variables are among the miRNAs `632' (promotes cell proliferation in carcinoma cancer) and `215' (early indicator of carcinoma cancer); `186' and `132' (both are colorectal cancer suppressants); and `374' and `9-1' (both are prostate cancer suppressants). Further, the majority of the edges in the graphical model with latent variables are among miRNAs that have different functionalities. Specifically, the five strongest edges in this graph are between the pairs: `449b' (breast cancer suppressant) and `577' (lung cancer suppressant); `192' (oncogene for prostate cancer) and `518c' (inhibits gastric cell growth); `449b' and `518c' (both are breast cancer suppressants), `449b' (breast cancer suppresant) and '143' (down-regulated in lung cancer); and '518c' (inhibits gastric cell growth) and '141' (biomarker in prostate cancer).  Of these five strongest edges, only the one linking `449b' and `518c' is between similar functioning miRNAs, and this edge is also present in the graphical model without latent variables.  In summary, these observations suggest that in the latent variable graphical model, the latent variables may correspond to commonalities in the biological functions of various miRNAs, and the associated confounding edges are not present in the conditional graphical model.  We observe a similar feature in the experimental results in the next subsection with stock return data.

\begin{figure}[ht!]
\centering
\includegraphics[scale = 0.51]{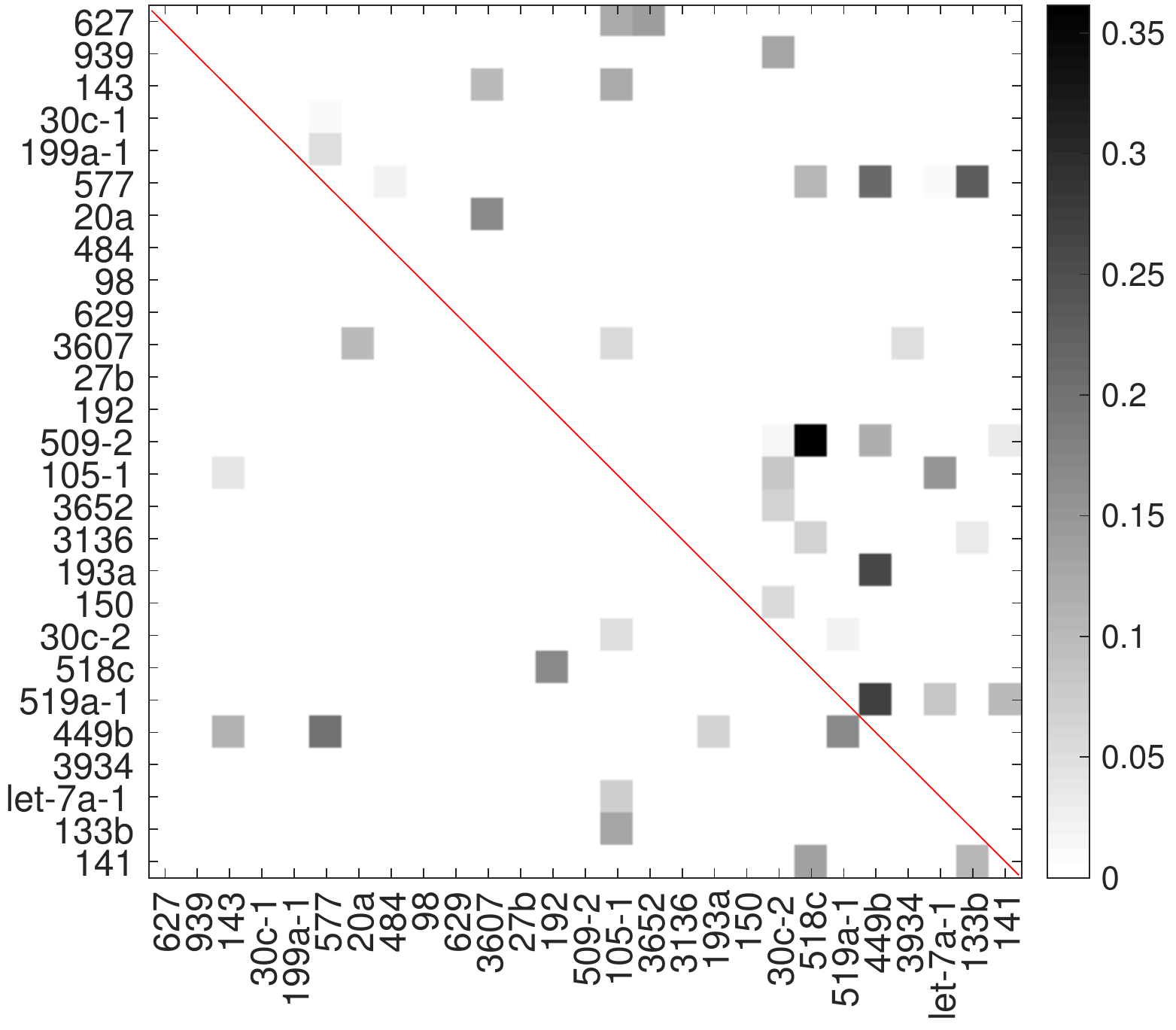}
\caption{Edges between miRNAs in the graphical model conditioned on two latent variables (bottom triangle) and in the graphical model without latent variables (top triangle).}
\label{fig:voter}
\end{figure}
%\end{SCfigure}

\subsection{Stock data}
We analyze monthly stock returns of $d = 40$ companies from the Standard
and Poor index over the period March 1982 to March 2016, which leads to a total of $n = 408$
observations. We set aside $n_\texttt{test} = 58$ observations as the test set, and the remaining $n_\texttt{train} = 350$ observations as the training set. In this experiment, we apply the convex relaxation \eqref{eq:gaussianclvgm} for fitting Gaussian graphical models conditioned on latent variables. 

We obtain a latent-variable graphical model with $r = 3$ latent variables and a conditional graphical model with edge density $\approx{3.8}\%$. The magnitudes of the partial correlations corresponding to this conditional graphical model are displayed in the bottom triangle of Figure~\ref{fig:stock}. The strongest five edges in this graph are between companies Andarko - Lowe, United Health - Intel, Bank of America - Cisco, IBM - Amgen, Fedex - Raytheon.  Note that in the Standard Industrial Classification system\footnote{See the U.S. SEC website at \url{http://www.sec.gov/info/edgar/siccodes.html}.} for grouping these companies, all of these pairs are in different classes. We also obtain a graphical model that does not incorporate latent variables, and the graph structure of this model has edge density $\approx 6.5\%$; the corresponding magnitudes of the partial correlations are shown in the top triangle of Figure~\ref{fig:stock}. In contrast to the previous model that incorporated latent variables, four of the five strongest edges in this graphical model without latent variables graph are between companies in the same category: Texas Instruments - Intel, HPQ - Microsoft, Wellsfargo - Bancorp, and
Boeing - General Dynamics. The negative log likelihoods (based on the procedure described previously in this section) evaluated on the test data yield values of $18.1$ for the model that incorporates latent variables as compared to $22.2$ for the model without latent variables, which suggests that accounting for the confounding effects of latent variables yields a better fit to stock data.

%\begin{SCfigure}
\begin{figure}
\centering
\includegraphics[scale = 0.57]{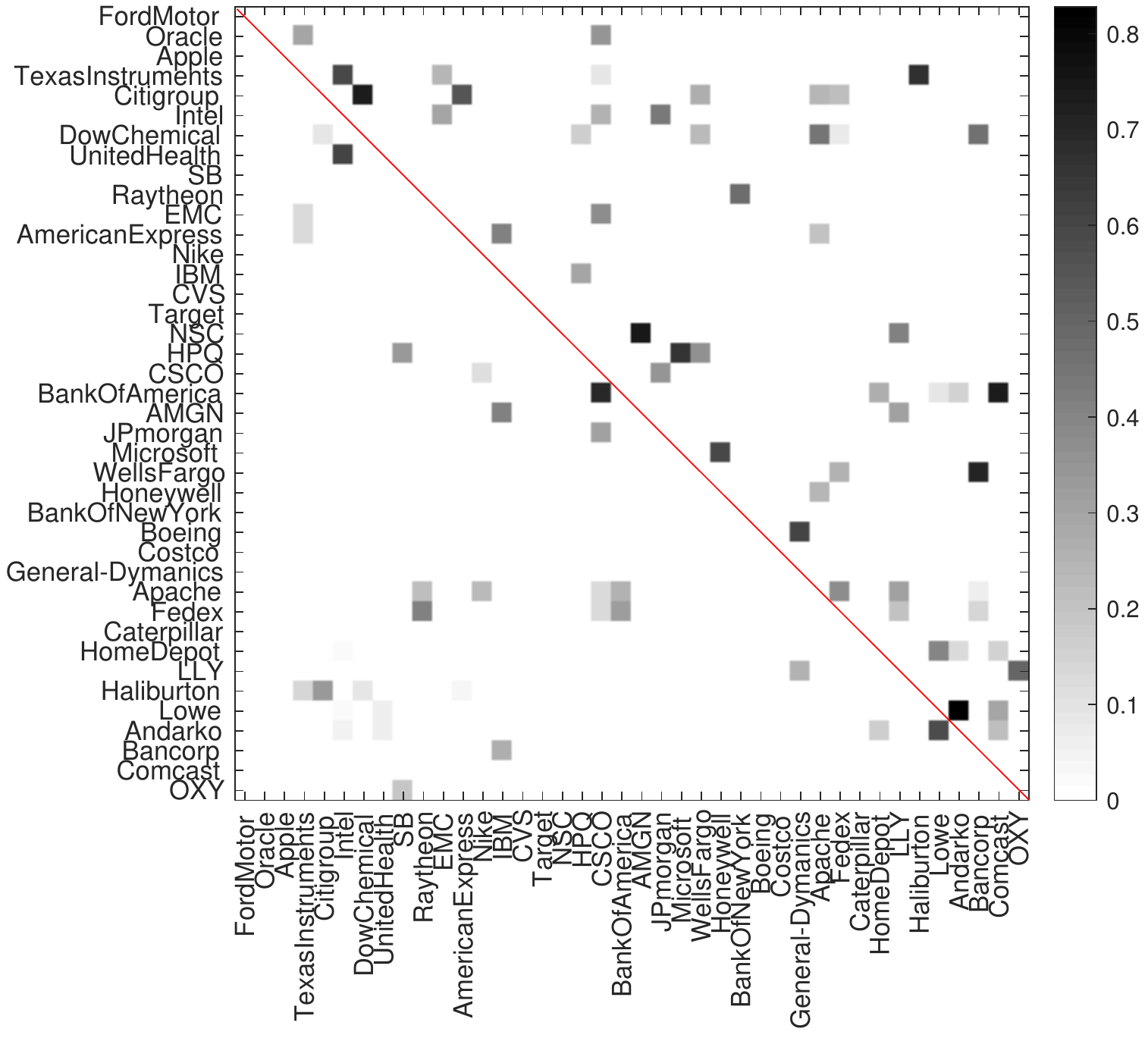}
%\subfigure[marginal approach - 3 latent variables]{
%\includegraphics[scale = 0.4]{stock_marginal.pdf}}
\caption{Bottom triangle: edges among pairs of companies obtained with the conditional likelihood procedure; top triangle: graphical model without incorporating latent variables; white indicates no edge.}
\label{fig:stock}
\end{figure}
%\end{SCfigure}

\section{Discussion}
\label{section:discussion}

In this paper, we describe a new convex relaxation framework for learning a latent variable graphical model given sample observations of a collection of variables.  Specifically, we fit the observations to a model in which the conditional distribution of the observed variables conditioned on latent variables is given by an exponential family graphical model \eqref{eq:expfam}.  Our approach is based on regularized conditional likelihood and we demonstrate the utility and flexibility of our method with both synthetic data as well as real data from a variety of problem domains.

There are several interesting directions for further investigation arising from our work that we outline below:

{\bf \par{Consistency of the estimators \eqref{eq:gaussianclvgm} and \eqref{eq:pseudoclvgm}}}: In Section~\ref{section:phase}, we presented empirical evidence that the estimators \eqref{eq:gaussianclvgm} and \eqref{eq:pseudoclvgm} consistently identify the structure of a latent variable graphical model in various settings and with several types of distributions.  Given this empirical demonstration, it would be of interest to provide theoretical support for the consistency of our method by leveraging prior work such as \cite{Chand2012}.

{\bf\par{Comparison of \cite{Chand2012} and estimator \eqref{eq:gaussianclvgm} for Gaussian graphical modeling:}} The authors in \cite{Chand2012} develop a convex relaxation for the problem of latent variable graphical modeling in settings with jointly Gaussian observed and latent variables. Their approach proceeds by considering the marginal distribution of the observed variables and explicitly characterizing the influence of the latent variables on the observed variables upon marginalization.  This characterization and the underlying assumption of joint Gaussianity are central to the derivation of the relaxation in \cite{Chand2012}.  In contrast, the derivation of our estimator \eqref{eq:gaussianclvgm} requires no knowledge of the distribution of the latent variables, and only assumes that the conditional distribution of the observed variables conditioned on the latent variables is given by a Gaussian graphical model.  This distinction suggests that the framework in this paper is more flexible and may be more robust to different distributions for the latent variables.  A natural question is to develop further theoretical and empirical understanding of comparative advantages of each approach for latent variable Gaussian graphical modeling.

{\bf\par{Theoretical support for the proposed model selection procedure:}} In Section~\ref{sec:model_sel}, we described a model selection procedure -- based on a notion of \emph{stability} -- for selecting regularization parameters $\lambda,\gamma$ in \ref{eq:gaussianclvgm} and in \ref{eq:pseudoclvgm} as well as for identifying suitable model structures.  Our approach is based on combining the ideas of \cite{Wasserman,stability} for the graphical model component and \cite{false_discovery} for the latent subspace. The empirical demonstrations in Section~\ref{section:experimental_model} suggest that our procedure provides good control over both Type-I and Type-II errors in estimating a population graphical model.  These are perhaps to be expected based on the theoretical analyses in \cite{stability,false_discovery,Wasserman}, and it would be useful to combine and formalize these results in our context by showing that the Type-I and Type-II errors can be provably controlled under appropriate assumptions.

{\bf\par{Better regularizers:}}
The nuclear norm regularizer in \eqref{eq:clvgm} (and its specializations) is agnostic to the type or form of the latent variables and only encourages low-rankness (i.e., few latent variables).  If a data analyst has access to additional information about potential latent variables (e.g., the latent variables take on non-negative or categorical values) or wishes to fit to models in which the latent variables have additional structure, one can design tighter convex regularizers than the nuclear norm \cite{ConvexGeom}.  As an example, if the latent variables take on binary values, a tighter regularizer than the nuclear norm is the max-2 norm.  Thus, an exciting direction is to investigate the computational and statistical tradeoffs underlying these tighter regularizers for latent variable graphical modeling.

{\bf\par{Tailored computational methods:}} We solve the convex program \eqref{eq:pseudoclvgm} via an ADMM procedure that we implemented ourselves \cite{ADMM}. The most costly component of this algorithm is computing a singular-value decomposition of $d \times n$ matrices, which can be prohibitive when the sample size or the number of variables are large.  For the convex program \eqref{eq:gaussianequiv}, we use the off-the-shelf logDetPPA solver \cite{Toh} and it tends to be prohibitively expensive beyond $d \approx 500$ on standard contemporary workstations.  Fast solvers for the graphical Lasso (and its variants) that can handle up to tens of thousands and sometimes millions of variables by exploiting problem-specific structure have been proposed previously \cite{Ravikumar_million,Ma}.  Designing similar custom solvers for the relaxations \eqref{eq:pseudoclvgm} and \eqref{eq:gaussianequiv} proposed in this paper would enable a broader application of our methods in problems with a large number of variables. 

\bibliographystyle{siam} % Style BST file (imsart-number.bst or imsart-nameyear.bst)
\bibliography{example}

\end{document}